\documentclass{article}
\usepackage[margin=1.25in]{geometry}
\usepackage{microtype}
\usepackage{graphicx}
\usepackage{subfigure}
\usepackage{booktabs} 
\usepackage{tabularx}
\usepackage{hyperref}
\usepackage[utf8]{inputenc} 
\usepackage[T1]{fontenc}         
\usepackage{url}          
\usepackage{nicefrac}       
\usepackage{algorithm}
\usepackage[noend]{algorithmic} 
\usepackage{times}
\usepackage{latexsym}
\usepackage{graphicx}
\usepackage[english]{babel} 
\usepackage{amsmath,amsfonts,amsthm} 
\usepackage{amssymb}
\usepackage{enumitem}
\usepackage{thmtools}
\usepackage{nameref,hyperref,cleveref,}
\usepackage{caption}
\usepackage{bm}
\usepackage{authblk}
\usepackage{xcolor}
\usepackage{chemarrow}
\usepackage{tikz-cd}
\usepackage{multicol}

\input{definitions}

\newtheorem{theorem}{Theorem}

\newtheorem*{theorem*}{Theorem}

\declaretheorem[name=Definition, sibling=theorem, refname={definition,definitions}, Refname={Definition,Definitions}]{definition}

\title{Extreme Multi-label Classification from Aggregated Labels}

\author[$\dagger$]{Yanyao Shen}
\author[$\ddagger$]{Hsiang-fu Yu}
\author[$\dagger\ddagger$]{Sujay Sanghavi}
\author[$\ddagger\mathsection$]{Inderjit Dhillon}
\affil[$\dagger$]{Department of ECE,  The University of Texas at Austin}
\affil[$\ddagger$]{Amazon}
\affil[$\mathsection$]{Department of Computer Science, The University of Texas at Austin}
\date{}
\begin{document}
	
\maketitle

\begin{abstract}
Extreme multi-label classification (XMC) is the problem of finding the relevant labels for an input, from a very large universe of possible labels. We consider XMC in the setting where labels are available only for groups of samples - but not for individual ones. Current XMC approaches are not built for such multi-instance multi-label (MIML) training data, and MIML approaches do not scale to XMC sizes. We develop a new and scalable algorithm to impute individual-sample labels from the group labels; this can be paired with any existing XMC method to solve the aggregated label problem. We characterize the statistical properties of our algorithm under mild assumptions, and provide a new end-to-end framework for MIML as an extension. Experiments on both  aggregated label  XMC and MIML tasks show the advantages over existing approaches.
\end{abstract}

\section{Introduction}
Extreme multi-label classification (XMC) is the problem of finding the relevant labels for an input from a very large universe of possible labels. 
XMC has wide applications in machine learning including product categorization~\cite{agrawal2013multi,yu2014large}, webpage annotation~\cite{partalas2015lshtc} and hash-tag suggestion~\cite{denton2015user}, 
where both the sample size and the label size are extremely large. 
Recently, many XMC methods have been proposed with new benchmark results on standard datasets~\cite{prabhu2018extreme,guo2019breaking,jain2019slice}. 

XMC problem, as well as many other modern machine learning problems, often require a large amount of data. 
As the size of the data grows, the annotation of the data becomes less accurate, and 
large-scale data annotation with high quality becomes growingly expensive. 
As a result, modern machine learning applications need to deal with certain types of
weak supervision, including partial  but noisy labeling  and active labeling. 
These scenarios lead to exploration of advanced learning methods including
semi(self)-supervised learning, robust learning and  active learning. 

In this paper, we study a typical weak supervision setting for XMC  named  Aggregated Label eXtreme  Multi-label Classification (\problemname), where only aggregated labels are provided to a group of samples. 
\problemname~is of interest in many practical scenarios where directly annotated training data can not be extracted easily, which is often due to the way data is organized. 
For example, Wikipedia contains a set of annotated labels for every wiki page, and can be used by an XMC algorithm for the task of tagging a new wiki page. 
However, if one is interested in predicting keywords for a new wiki \textit{paragraph}, there is no such directly annotated data. 
Similarly, in e-commerce, the attributes of a product may not be provided directly, but the attributes of the brand of the product may be easier to extract. 
To summarize, it is often easier to get aggregated annotations belonging to a group of samples. This is  known as  multi-instance multi-label (MIML)~\cite{zhou2012multi} problem in the non-extreme label size setting.

\problemname~raises new challenges that standard  approaches are not able to address. 
Because of the enormously large label size, directly using MIML methods  leads to computation and memory issues. 
On the other hand, standard XMC approaches suffer from two main problems when directly applied to \problemname: (1) higher computation cost due to increased number of positive labels, and (2) worse performance due to ignoring of the aggregation structure. 
In this work, we propose an Efficient AGgregated Label lEarning algorithm (\algname) that assigns labels to each sample by learning  label embeddings based on the structure of the aggregation. 
More specifically, the key ingredient of \algname~follows the simple principle that  \textit{the label embedding should be close to the embedding of at least one of the sample points in every positively labeled group}.  
We first formulate such an estimator, then design an iterative algorithm that takes projected gradient steps to approximate it. 
As a by-product, our algorithm naturally extends to the non-XMC setting as a new end-to-end framework for the MIML problem. 
Our main contributions include:
\begin{itemize} 
	\item We propose to study  \problemname, which has significant impact for modern machine learning applications. We propose an efficient and robust algorithm \algname~with low computation  cost (\textbf{ Section \ref{sec:alg}}) that can be paired with any existing XMC method for solving \problemname.

	\item We provide theoretical analysis for \algname~and show the benefit of label assignment, the property of the estimator and the convergence of our iterative update (\textbf{Section \ref{sec:analysis}}).
	
	\item The proposed method can be easily extended to the regular (non-extreme) MIML  setting (\textbf{Section \ref{sec:alg_extension}}). Our solution can be viewed as a co-attention mechanism between labels and samples. We empirically show its benefit over previous MIML framework in \textbf{Section \ref{sec:exp}}. 
\end{itemize}

\section{Related Work}
\label{sec:related}

\paragraph{Extreme multi-label classification (XMC).}
The most classic and straightforward approach for XMC is the One-Vs-All (OVA) method~\cite{yen2016pd,babbar2017dismec,liu2017deep,yen2017ppdsparse}, which simply treats each label separately and learns a classifier for each label. 
OVA has shown to achieve high accuracy, but the computation is too expensive for extremely large label set. 
Tree-based methods, on the other hand, try to improve the efficiency of OVA by using hierarchical representations for samples~\cite{agrawal2013multi,prabhu2014fastxml,jain2016extreme,si2017gradient} or labels~\cite{prabhu2018parabel,jain2019slice}. Among these approaches, label partitioning based methods, including Parabel~\cite{prabhu2018parabel}, have achieved leading performances with training cost sub-linear in the number of labels. 
Apart from tree-based methods, embedding based methods~\cite{zhang2018deep,chang2019xbert,you2019haxmlnet,guo2019breaking} have been studied recently in the context of XMC in order to better use the textual features. 
In general, while embedding based methods may learn a better representation and use the contextual information better than tf-idf, the scalability of these approaches is worse than tree-based methods. 
Very recently,  Medini et al.~\cite{medini2019extreme} apply sketching to learn XMC models with  label size at the scale of $50$ million.

\paragraph{Multi-instance multi-label learning  (MIML).} 
MIML~\cite{zhang2007multi} is a general setting that includes both multi-instance learning (MIL)~\cite{dietterich1997solving,maron1998framework} and multi-label learning (MLL)~\cite{mccallum1999multi,zhang2013review}. \problemname~can be categorized as a special MIML setting with extreme label size. 
Recently, Feng and Zhou~\cite{feng2017deep} propose the general deep MIML architecture with  a `concept' layer and two max-pooling layers to align with the multi-instance nature of the input.  
In contrast, our approach learns label representations to use them as one branch of the input. 
On the other hand, Ilse et al. \cite{ilse2018attention} adopt the attention mechanism for multi-instance learning. 
Similar attention-based mechanisms are later used in learning with sets~\cite{lee2019set} but focus on a different problem. 
Our label assignment based algorithm \algname~can be viewed as an analogy to the attention-based mechanisms, while having major differences from previous work. 
\algname~provides the intuition that   attention   truly happens between the label representation and the sample representation, while previous methods do not.  
The idea of jointly considering sample and label space exists in the multi-label classification problems in vision~\cite{weston2011wsabie,frome2013devise}. 
While sharing the similar idea of learning a joint input-label space, our work addresses the multi-instance learning challenges as well as scalability in the XMC setting.

\paragraph{Others.} 
\problemname~is also related to a line of theoretical work on learning with shuffled labels and permutation estimation ~\cite{collier2016minimax,pananjady2017linear,abid2017linear,pananjady2017denoising,hsu2017linear,haghighatshoar2017signal}, where the labels of all samples are provided without correspondences. 
Our work uniquely focuses on an aggregation structure where we know the group-wise correspondences. 
Our targeted applications have extreme label size that makes even classically efficient estimators hard to compute. 
Another line of work studies learning with noisy labels~\cite{natarajan2013learning,liu2015classification}, where one is interested in  identifying the subset of correctly labeled samples~\cite{shen2019learning}, but there is no group-to-group structure. 
More broadly, in the natural language processing context, indirect supervision~\cite{chang2010structured,wang2018deep} tries to address the label scarcity problem where  large-scale coarse annotations are provided with very limited fine-grain annotations. 

\section{Problem Setup and Preliminaries}
\label{sec:formulation}

In this section, we first provide a brief overview of XMC, whose definition helps us   formulate  the  aggregated label XMC (\problemname) problem. 
Based on the formulation, we use one toy example to illustrate the shortcomings of existing XMC approaches when applied to \problemname. 

\paragraph{XMC and \problemname~formulation.} 
An XMC problem can be defined by $\{\Xb, \Yb\}$, where $\Xb\in \mathbb{R}^{n\times d}$ is the feature matrix for all $n$ samples,  $\Yb\in \{0,1\}^{n\times l}$ is the sample-to-label binary annotation matrix with label size $l$ (if sample $i$ is annotated by label $k$ then $\Yb_{i,k}=1$). 
For the \problemname~problem, however, such a clean annotation matrix is not available. 
Instead, aggregated labels are available for subsets of samples. We use   $m$ intermediate nodes to represent this aggregation, where each node is connected to a subset of samples and gets annotated by multiple labels. 
More specifically, \problemname~can be described by $\{\Xb, \Yb^1, \Yb^2\}$, where the original annotation matrix is replaced by two binary matrices $\mathbf{Y^1}\in \{0,1\}^{n\times m}$ and $\mathbf{Y^2}\in \{0,1\}^{m\times l}$. 
$\Yb^1$ captures how  the samples are grouped, while $\Yb^2$ captures the labels for the aggregated samples. 
The goal is to use $\{\Xb, \Yb^1, \Yb^2\}$ to learn a good extreme multi-label classifier. 
Let $\bar{g}={\mathtt{nnz}(\Yb^1)}/{m}$ denote the average group size. In general, the larger $\bar{g}$ is, the weaker the annotation quality becomes. 
For convenience, let $\Ncal, \Mcal, \Lcal$ be the set of samples, intermediate nodes, and labels, respectively. 
Let $\Ncal_j, \Lcal_j$ be the set of samples, labels linked to intermediate node $j$ respectively, $\forall j\in \Mcal$, 
and $\Mcal_k$ be the set of  nodes in $\Mcal$ connected to label $k\in \Lcal$\footnote{We slightly abuse the notation and use $\Ncal_j, \Mcal_k, \Lcal_j$ for the corresponding index sets as well.}. Let $\xb_{i}^\top$ be the $i$-th row in $\Xb$ for   $i\in \Ncal_j$, $j\in \Mcal$, and $\Xb_{\Scal}$ be the submatrix of $\Xb$ that includes rows with index in set $\Scal  \subseteq \Ncal$\footnote{We summarize all notations in the Appendix.}.

\begin{figure}[t]
	\centering
	\includegraphics[width =0.5\columnwidth ]{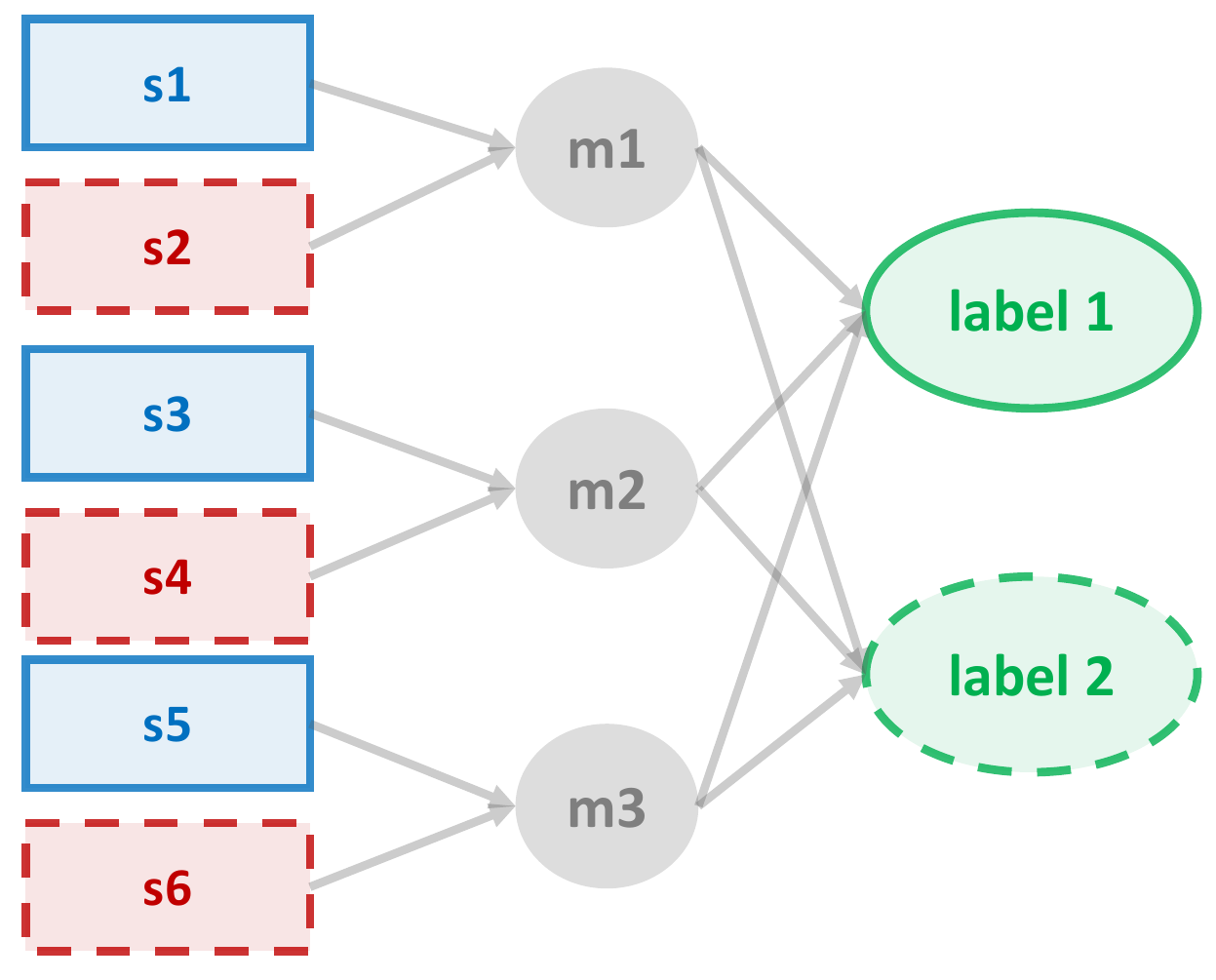}
	\caption{
	A toy example of \problemname.  
	Label 1 \& 2 are both tagged for each group in $\{s_1, s_2\}, \{s_3, s_4\}, \{s_5, s_6\} $. 
	Samples $s_1, s_3, s_5$ all have feature $\xb$,  while $s_2, s_4, s_6$ all have feature $-\xb$.  
	A good model should identify that $s_1, s_3, s_5$ and $s_2, s_4, s_6$ belong to two labels, respectively (order does not matter). 
	Ignoring the aggregation structure makes standard XMC approaches fail completely. 
	}
	\label{fig:illustration}
\end{figure}

\paragraph{Deficiency of existing XMC approaches.} 
Existing XMC approaches can be applied for \problemname~by treating the product $\Yb^1\Yb^2$ directly as the annotation matrix, and learning a model using $\{\Xb, \Yb^1\Yb^2 \}$. 
While using all labeling information, this simple treatment ignores the possible incorrect correspondences due to label aggregation. 
To see the problem of this treatment, we take the XMC method Parabel~\cite{prabhu2018parabel}  as an example. 
Notice that the deficiency generally holds for all standard XMC approaches, but it is convenient to illustrate for a specific XMC method. 
Parabel is a scalable   algorithm that achieves good performance on standard XMC datasets based on a partitioned label tree:  
 It first calculates label representations $\Lb = \Yb^\top \Xb$ that summarize the positive samples associated with every label (\textbf{step 1});  
 Next, a balanced hierarchical partitioning of the labels is learned using $\Lb$ (\textbf{step 2});  
 Finally, a hierarchical probabilistic model is learned given the label tree (\textbf{step 3}). 
Notice that the label embedding that Parabel would use for \problemname~if we naively use $\Yb^1\Yb^2$ as the annotation matrix is given by: 
\begin{align}
\mathbf{L} = \left( \mathbf{Y^1Y^2} \right)^\top \mathbf{X}.  \label{eqt:label-embedding-pifa} 
\end{align}
Consider an example where  there are $n$ samples and $2$ labels, with $\Xb \in \mathbb{R}^{n\times d}$,   $\Yb^1\in \{0, 1\}^{n \times \frac{n}{2}}$, $\Yb^2\in \{0,1\}^{\frac{n}{2} \times 2}$ defined as follows:
$$
\mathbf{X} = \mathbf{1}_{\frac{n}{2}}\otimes  \begin{bmatrix}
  \mathbf{x}^\top   \\ 
 - \mathbf{x}^\top  \\ 
\end{bmatrix},
\mathbf{Y^1} = \Ib_{\frac{n}{2}} \otimes \mathbf{1}_2, 
\mathbf{Y^2} = \mathbf{1}_{\frac{n}{2}\times 2}, 
$$ 
where $\otimes$ is the Kronecker product, $\mathbf{1}_d$($\mathbf{1}_{d_1\times d_2}$) is an all-ones vector(matrix) with dimension $d$($d_1\times d_2$) and $\Ib_d$ is the identity matrix with size $d$. 
A pictorial explanation is shown in Figure \ref{fig:illustration}. 
The embedding calculated using the above $\mathbf{X, Y^1, Y^2}$ leads to $\mathbf{L} = \mathbf{0}_{n\times 2}$ and loses all the information. With this label embedding, the clustering algorithm in step 2 and the probabilistic model in step 3 would fail to learn anything useful. 
However, a good model for the above setting should classify samples with feature close to $\xb$ as label $1$ and samples with feature close to $-\xb$ as label $2$ (or vice versa). 
Such a failure of classic XMC approaches motivates us to provide algorithms that are robust for the \problemname~problem.

\section{Algorithms}
\label{sec:alg}

The main insight we draw from the toy example above is that ignoring the aggregation structure may lead to serious information loss. 
Therefore, we propose an efficient and robust label embedding learning algorithm to address this. 
We start with the guiding principle of our approach, and then explain the algorithmic details. 

Given an XMC dataset with aggregated labels, our key idea for finding the embedding for each label is the following:
\begin{quote}
\textit{  The embedding of label $k\in \Lcal$ should be close to the embedding of \textbf{at least one} of the samples in $\Ncal_{j}$, $\forall j\in \Mcal_k$. }
\end{quote}

The closeness here can be any general characterization of similarity, e.g., the standard cosine similarity. 
According to this rule, in the previous toy example, the optimal label embedding for both labels is either $[\mathbf{x}, -\mathbf{x}]$  or $[-\mathbf{x}, \xb]$, instead of $[\mathbf{0}, \mathbf{0}]$.  
More formally, the label embedding for label $k\in \Lcal$ is calculated based on the following:
\begin{align}\label{eqt:label_embedding}
	\hat{\eb}_{k} = \arg\max_{\mathbf{e}: \|\mathbf{e}\| = 1} \sum_{j \in \Mcal_{k}} \max_{i\in \Ncal_{j}} \similarity{   \mathbf{x}_{i}}{ \mathbf{e}  }. 
\end{align}
where $\Mcal_{k}, \Ncal_j$ are as defined in Section \ref{sec:formulation}.

\renewcommand{\algorithmiccomment}[1]{\textit{#1}} 

\begin{algorithm}[tb]
\caption{\textsc{Group\_Robust\_Label\_Repr} (\textsc{GRLR})}
\label{alg:1}
\begin{algorithmic}
\STATE \textbf{Inputs:}  $\Mcal_k$, $ \{ \Ncal_j\}_{j\in \Mcal_k}$, $\Xb$.
\STATE \textbf{Output:} Label embedding $\eb_k$. 
\STATE \textbf{Initialize:} Set ${\eb}^0 \leftarrow \mathtt{Proj}\left(  \sum_{j\in \Mcal_k} \sum_{i\in \Ncal_j} \xb_{i} \right) $. 
\FOR[\hfill \textit{\textcolor{gray}{ /* where  $\mathtt{Proj}(\xb): = \xb/\|\xb\|$ */ }}]{$t =1, \cdots, T$} 
    \FOR{$j\in \Mcal_k$}
	\STATE $v_{i,j} \leftarrow \langle {\eb}^{t-1}, \xb_{i}\rangle $, $\forall i \in \Ncal_j$.
	\STATE $a_{i,j}\! \leftarrow\! \mathbf{1} \left\{ v_{i,j} == \max_{i^\prime \in \Ncal_j } v_{i^\prime,j} \right\}$, $\forall i \in \Ncal_j$.
	\ENDFOR
	\STATE $\gb^t \leftarrow \mathtt{Proj}\left( \sum_{j\in \Mcal_k} \sum_{i\in \Ncal_j} a_{i,j} \xb_{i} \right)$.
	\STATE $\eb^t \leftarrow \mathtt{Proj} \left(  \eb^{t-1} + \lambda \cdot  \gb^t  \right) $.
\ENDFOR
\STATE \textbf{Return:} $\eb^T$.
\end{algorithmic}
\end{algorithm}

The goal of (\ref{eqt:label_embedding}) is to find label $k$'s representation that is robust even when every group has samples not related to $k$.  
However, finding the estimator in (\ref{eqt:label_embedding}) is in general hard, since the number of possible combinations for choosing the maximum in each group is exponential in $n$. 
We provide an iterative algorithm that alternates between the following two steps for $T$ times to approximate this estimator: 
(i) identify the closest sample in each group given the current label embedding, and 
(ii) update the label embedding based on the direction determined by all currently closest samples. 
This is formally described in Algorithm \ref{alg:1}.  

The complete algorithm Efficient AGgregated Label lEarning (\algname)  is formally described in Algorithm \ref{alg:2}, whose output can be directly fed into any standard XMC solver. 
Given the label embedding and a set of samples connected to the same intermediate node, each positive label is assigned to the sample with highest similarity. 
Notice that calculating the label embedding using Algorithm \ref{alg:1} is \textit{equivalent} to Parabel's label embedding in (\ref{eqt:label-embedding-pifa}) if $\bar{g}=1$, i.e., the standard XMC setting. 
For general $\bar{g},  $(\ref{eqt:label-embedding-pifa}) may perform well if  each sample in $\Ncal_j$ contributes equally to the label $k\in \Lcal$, for every node $j\in \Mcal_k$. 
However, this is not always the case. 

\paragraph{Complexity.} 
Computational efficiency is one of the main benefits of \algname.  Notice that for each label, only the samples belonging to its positively labeled groups are used. 
Let $\bar{d}$ be the average feature sparsity and assume each sample has $O(\log l)$ labels, the positive samples for each label is $O(n \log l/ l \cdot \bar{g})$.  Therefore, the total complexity for learning all label's embedding  is  $O( n\bar{d} \log l/l \cdot  \bar{g} l) = O(n\bar{d}\bar{g} \log l)$. 
On the other hand, the time complexity for Parabel (one of the most efficient XMC solver) is $O(n\bar{d}\log l)$ for step 1, $O(l\bar{d}\log l)$ for step 2 and  $O(n\bar{d}\log l)$ for step 3~\cite{prabhu2018parabel}. 
Therefore, \algname~paired with any standard XMC solver for solving \problemname~adds very affordable pre-processing cost. 

\begin{algorithm}[tb]
	\caption{\textsc{\algname}}
	\label{alg:2}
	\begin{algorithmic}
		\STATE \textbf{Inputs:} $\Xb, \Yb^1\in \{0,1\}^{n\times m}, \Yb^2\in \{0,1\}^{m\times l}$.
		
		\STATE \textbf{Output:} A filtered XMC dataset.
		\STATE $\Yb_{\texttt{filter} } \leftarrow \mathbf{0}^{n\times l}$, $\Ncal\leftarrow [n], \Mcal \leftarrow [m], \Lcal\leftarrow [l]$.
		\STATE  $\Ncal_j\leftarrow \{ i\in \Ncal | \Yb^1_{i,j}==1 \}$, $\forall j\in \Mcal$.
		\FOR {$k \in \Lcal $ }
			\STATE $\Mcal_k \leftarrow \{ j\in \Mcal | \Yb^2_{j, k}==1 \}$.
			\STATE $\eb_k \leftarrow $\textsc{GRLR}$(\Mcal_k,  \{ \Ncal_j\}_{j\in \Mcal_k}, \Xb)$.
		\ENDFOR
		\FOR {$j\in \Mcal$ }
			\STATE $\Yb^{\texttt{filter} }(\arg\max_{i\in \Ncal_j} \langle\eb_k, \xb_{i}\rangle, k) \leftarrow 1$,  $\forall k \in \Lcal_j$.
		\ENDFOR
		\STATE \textbf{Return:} $\{\Xb, \Yb^{\texttt{filter}}\}$. 
	\end{algorithmic}
\end{algorithm}
 
\section{Analysis}
\label{sec:analysis}

In this section, we provide theoretical analysis and explanations to the proposed algorithm \algname~in Section \ref{sec:alg}. 
We start with comparing two estimators under the simplified regression setting to explain  when assigning labels to each sample is helpful. 
Next, we analyze the statistical property of the label embedding estimator defined in (\ref{eqt:label_embedding}) in Theorem \ref{thm:2}, and the one-step  convergence result of the key step in Algorithm \ref{alg:1} in Theorem \ref{thm:3}. 

In \algname, a learned label embedding is used to assign each label to the `closest' sample in its aggregated group.  
Therefore, we start with justifying when label assignment would help. 
Since the multi-label classification setting may complicate the analysis, we instead analyze a simplified regression scenario. 
Let $\Zb \in \mathbb{R}^{n\times l}$ be the response of all $n$ samples in $\Xb$. 
Given $\Bb^\star \in \mathbb{R}^{d\times l}$, each group in $\Zb$ is generated according to 
$$
\Zb_{\Ncal_j} = \Pib^j (\Xb_{\Ncal_j}\Bb^\star + \Eb^j) , j\in \Mcal, 
$$
where $\Eb^j$ is the noise matrix, and $\Pib^j$ is an unknown permutation matrix. 
For simplicity, we assume each group includes $g$ samples and the aggregation structure can be described by $\Yb^1 = \Ib_{m}\otimes \mathbf{1}_{g}$,  ${\Yb}^2 =\left( \Ib_{m}\otimes \mathbf{1}_{g}^\top\right) \cdot  \Zb$ with $m=n/g$. 
If each row in $\Zb$ is a one-hot vector,  $\Yb^2$ becomes a binary matrix and $\{ \Xb, \Yb^1, \Yb^2\}$ corresponds to the standard \problemname~problem. 
Our goal here is to recover the model parameter $\Bb^\star$, with $\|\Bb^\star\|=1$ for convenience. 
We assume each row in $\Eb^j$ independently follows $\Ncal(\mathbf{0}, \sigma_e^2 \Ib_l)$, each sample feature is generated according to  $\xb_i = \bar{\xb}_j + \db_i$ for $i\in \Ncal_j, j\in \Mcal$, where $\bar{\xb}_j\sim \mathcal{N}(\mathbf{0}, \sigma_1^2 \Ib_d)$ describes the center of each group, and $\db_i \sim \mathcal{N}(\mathbf{0}, \sigma_2^2 \Ib_d)$ captures the deviation within the group. 
Notice that the special case of $\sigma_1 \ll \sigma_2$ corresponds to all samples are i.i.d. generated spherical Gaussians. 
In the other extreme, $\sigma_2 \ll \sigma_1$ corresponds to samples within each group are clustered well. 
We consider the following two estimators: 
\begin{align}
\hat{\Bb}_{\mathtt{\estimatora}} =& \mathtt{LR}\left( \cup_{j\in \Mcal} \left\{ \left( \mathbf{1}_g^\top \Xb_{\Ncal_j}, \mathbf{1}_g^\top \Zb_{\Ncal_j}  \right) \right\} \right) \nonumber \\
\hat{\Bb}_{\mathtt{\estimatorb}} =& \mathtt{LR} \left(  \cup_{j\in \Mcal} \cup_{i\in \Ncal_j} \left\{ \left( \xb_{i^\prime}, \zb_{i} \right) \right\} \right) \label{eqt:estimators} 
\end{align}
where $\mathtt{LR}\left(\{ (\xb_i, \zb_i) \}_{i\in [n]}\right) \!=\!  \left( \sum\limits_{i\in [n]} \xb_i\xb_i^\top \right)^{-1}\! \sum\limits_{i\in [n]} \xb_i\zb_i^\top$ 
and $i^\prime =  \arg\min_{\bar{i}\in \Ncal_j}  \left\| \zb_i -  \hat{\Bb}_{\mathtt{\estimatora}}^\top  \xb_{\bar{i}} \right\|$. 
Here, $\hat{\Bb}_{\mathtt{\estimatora}}$ corresponds to the baseline approach that learns a model without label assignment. 
On the other hand,  $\hat{\Bb}_{\mathtt{\estimatorb}}$ corresponds to the estimator we learn after  assigning each output in the group to the closest instance based on the residual norm using $\hat{\Bb}_{\mathtt{\estimatora}}$. 
We have the following result that describes the property of the two estimators: 

\begin{theorem}\label{thm:1}
	Given the two estimators in (\ref{eqt:estimators}), let $\Rcal_{1}= { \left\| \hat{\Bb}_{\mathtt{\estimatora}} - \Bb^\star \right\|}$, $\Rcal_{2}= {\left\| \hat{\Bb}_{\mathtt{\estimatorb}} - \Bb^\star \right\|} $, $\sigma_{x} = \sqrt{\sigma_1^2 + \sigma_2^2}$, with $n\ge  {c}_0 pgd\log^2 d$, the following holds with high probability (i.e., $1-n^{- {c}_1}$): 
	\begin{align} \label{eqt:thm-1-1} 
	    \Rcal_{1} \le O\left(\sqrt{\frac{1}{p(g\sigma_1^2 + \sigma_2^2)}  } \sigma_e \right),
	\end{align}
	\begin{align}\label{eqt:thm-1-2}
	    \Rcal_{2}\! \le \!  O\left(\!\sqrt{\frac{1}{pg\sigma_x^2}} \sigma_e\!\right)  \!+\! O\left(\!\! 
	    \sqrt{ {\frac{\sigma_e^2}{\sigma_x^2} +  \Rcal_{1}^2 }}
	    \sqrt{  \frac{\sigma_e^2}{\sigma_x^2}  + 1} 
	    \!\right).
	\end{align}
\end{theorem}
	
We can see the pros and cons of the two estimators from the above theorem. 
The first term in (\ref{eqt:thm-1-2}) is the rate achieved by the maximum likelihood estimator with all correspondences given (known $\Pib^j$s), 
while the second term is a bias term due to incorrect assignment. 
This bias term gets smaller as the measurement noise and the estimation error become smaller. 
In (\ref{eqt:thm-1-1}),  when $\sigma_1 \gg \sigma_2$, $\hat{\Bb}_{\mathtt{\estimatora}}$ achieves the same rate as the optimal estimator, but when $\sigma_1 \ll \sigma_2$, the rate goes down from $n^{-1/2}$ to $(n/g)^{-1/2}$. 
This shows that $\hat{\Bb}_{\mathtt{\estimatora}}$ is close to optimal when the clustering quality is high (within group deviation is small), on the other hand, $\hat{\Bb}_{\mathtt{\estimatorb}}$ is nearly optimal for all clustering methods, while having an additional bias term that depends on the measurement noise. 

Next, we analyze the property of the estimator in (\ref{eqt:label_embedding}), since our iterative algorithm tries to approximate it. 
We assume that for each label $k\in\Lcal$, there is some ground truth embedding $\eb_{k}^\star$, and each sample is associated with one of the labels.
For sample $i$ with ground truth label $k$, its feature vector $\xb_i$ can be decomposed as: 
$
\xb_i = \eb_{k}^\star + \epsilonb_{i}. 
$
Without loss of generality, we only need to focus on the recovering of single label $k\in \Lcal$. 
Further, for simplicity, let us assume that both $\xb_i$ and $\eb_{k}^\star$ have unit norm measured in Euclidean space. 
We introduce the following definition that describes the property of the data:
\begin{definition}\label{def:1}
	Define $\delta=\underset{k_1, k_2}{\min} \left\| \eb_{k_1}^\star - \eb_{k_2}^\star \right\|$ to be the minimum separation between each pair of ground truth label embeddings. 
	Define $f(\gamma)= \underset{\Scal \subset \Mcal, |\Scal|/|\Mcal|\le \gamma}{\max} \frac{1}{|\Scal|}\left\| \sum_{i\in \Scal} \epsilon_i\right\| $ to be the maximum influence of the noise, for $\gamma \in [0,1]$, and let $f=f(1)$. Define $q=\underset{k_1, k_2}{\max} \frac{|\Mcal_{k_1}\cap \Mcal_{k_2}|}{\min\{ \Mcal_{k_1}, \Mcal_{k_2} \}}$ to be the maximum overlap between the set of intermediate nodes associated with two labels. 
\end{definition}
Given the above definition, we have the following result showing the property of the estimator in (\ref{eqt:label_embedding}): 
\begin{theorem}
	\label{thm:2}
	With $q, \delta > 0$, the estimator in (\ref{eqt:label_embedding}) satisfies:
	\begin{align}\label{eqt:thm-2}
	\similarity{\eb_k^\star}{\hat{\eb}_k}   	\ge & 1 - r  f  - (\sqrt{2} r + 2) f^2,
	\end{align}
	where $r = \left( \left[ {1- q - \frac{2 f }{\delta} } \right]_{+}\right)^{-1}-1$. 
\end{theorem}
Theorem \ref{thm:2} characterizes the consistency of the estimator as noise goes to zero. It also quantifies the influence of minimum separation as well as maximum overlap between labels. A smaller $\delta$ and a larger $q$ both leads to harder identification problem, which is reflected in an increasing $r$ in (\ref{eqt:thm-2}). 
We then provide the following one-step convergence analysis for each iteration in Algorithm \ref{alg:1}.

\begin{figure}[t]
	\centering
	\includegraphics[width = \columnwidth ]{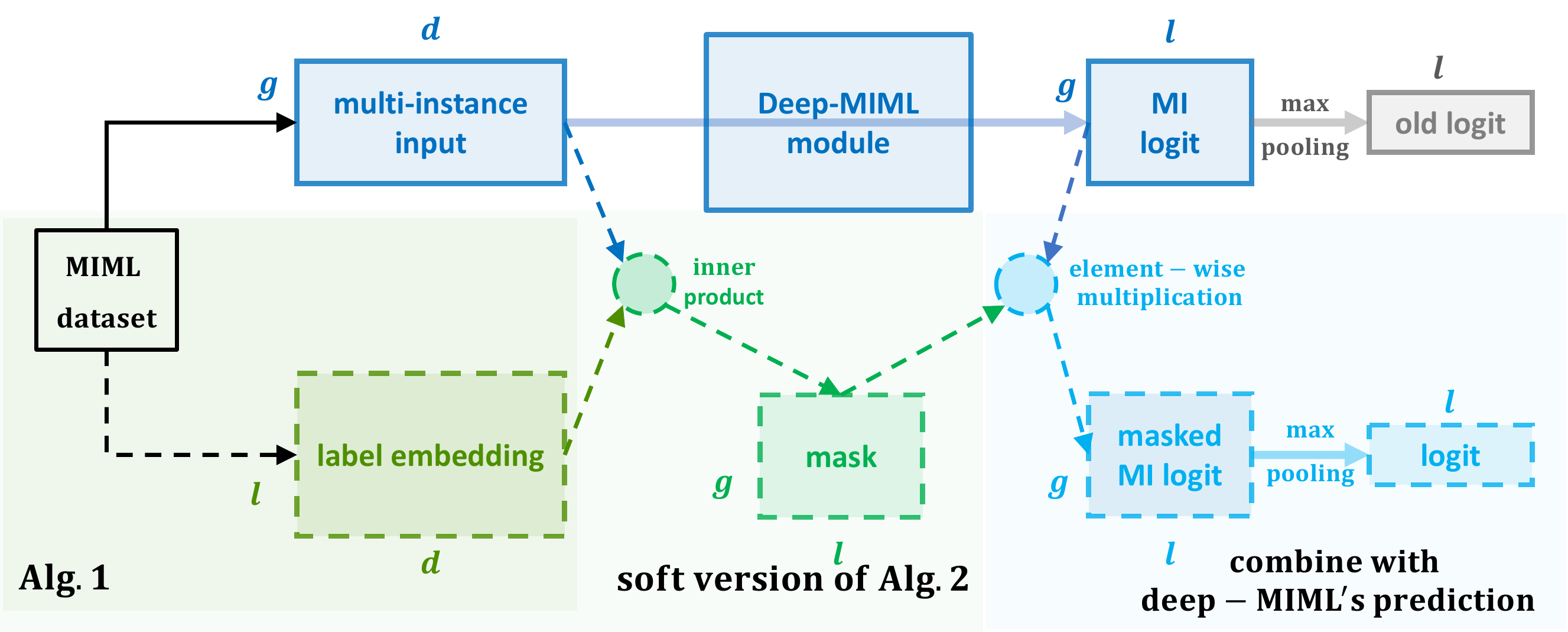}
	\caption{Illustration of extending \algname~to standard MIML framework. The upper branch is the original deep MIML framework \cite{feng2017deep}, where the Deep-MIML module converts $\Xb$ to $\Vb_2$ as in (\ref{eqt:dmiml}). 
	In the lower branch, 
		\algname~first learns label embedding using Algorithm \ref{alg:1} by re-organizing the MIML dataset. Then, a soft-assignment mask is generated based on the inner product between the multi-instance inputs and the label embedding. The masked MI logit is an element-wise multiplication of the mask and the original MI logit. A max-pooling operation over instances on the masked MI logit gives the final logit for prediction, similar to the original deep MIML. 
	}
	\label{fig:miml}
\end{figure}

\begin{theorem}
	\label{thm:3}
	Given label $k$ and current iterate $\eb^t$. 
	Let $\Scal_t^\good$ be the set of groups where $\eb^t$ is closest to the sample belongs to label $k$. 
	Denote $|\Scal_t^\good| / |\Mcal_k| = \alpha_t$. The next step iterate given by Algorithm \ref{alg:1} has the following one-step  property: 
	\begin{align}
		\similarity{\eb_k^\star }{\eb^{t+1}} \ge& \alpha_t + (1-\alpha_t) \left( \similarity{\eb_k^\star}{\eb^{t}} - \norm{\eb_k^\star - \eb^t}\right) 
		- f \nonumber 
	\end{align}
	and a sufficient condition for contraction 
	is $\similarity{\eb_k^\star}{\eb^{t}} \le 1 - 2 \left( \frac{1-f/2}{\alpha_t} -1 \right)^2$. 
\end{theorem}
Theorem \ref{thm:3} shows how each iterate gets closer to the ground truth label embedding. 
Since our algorithm does not require any assumption on the group size, we do not show the connection between $\alpha_t$ and $\eb^t$ (which requires more restrictive assumptions), but instead provide a sufficient condition to illustrate when the next iterate would improve.  
Notice that as the group size becomes larger, the signal becomes smaller and $\alpha$ is smaller in general.

\section{Extensions}
\label{sec:alg_extension}

In the previous section, \algname~is proposed for \problemname~in the extreme label setting. In the non-extreme case, \algname~naturally leads to a solution for the general MIML problems. 

\paragraph{Deep-MIML Network.} Feng and Zhou \cite{feng2017deep} propose a general deep-MIML network: given a multi-instance input $\Xb$ with shape $g\times d$, the network first transforms it into a $g\times k \times l$ tensor $\Vb_1$ through a fully connected layer and a ReLU layer, where $l$ is the label size and $k$ is the additional dimension called `concept size' \cite{feng2017deep} to improve the model capacity. A max-pooling operation over all `concepts' is taken on $\Vb_1$ to provide a $g\times l$ matrix $\Vb_2$ (which we call multi-instance logit). Finally, max-pooling over all instances is taken on $\Vb_2$ to give the final length-$l$ logit prediction $\hat{\Yb}$. This network can be summarized as: 
\begin{align}
	\Xb \overset{\mbox{FC + ReLU} }{\rarrowfill{1cm}}  \Vb_1 \underset{\mbox{(over concepts)} }{\overset{\mbox{max-pooling} }{\rarrowfill{1.5cm}}}  \Vb_2 \underset{\mbox{(over samples)} }{\overset{\mbox{max-pooling} }{\rarrowfill{1.5cm}}}   \hat{\Yb}
\label{eqt:dmiml}
\end{align}

\paragraph{A co-attention framework} Our idea can be directly applied to modify the deep-MIML network structure, as shown in  
Figure \ref{fig:miml}. 
The main idea is to add a soft-assignment mask to the original multi-instance logit, where this mask mimics the label assignment in \algname. 
After learning the label embedding $\Lb\in \mathbb{R}^{l\times d}$ from the dataset using Algorithm \ref{alg:1}, the mask $\Mb\in \mathbb{R}^{g\times l}$ is calculated by $\Mb=g\cdot \mathtt{Softmax}\left(\tau \Xb_{\Mcal_j} \Lb^\top\right)$ where this softmax operation applies to each column in $\Mb$. As a result, $\Mb_{i,j}$ indicates the affinity between instance $i$ and label $j$. 
Notice that $\tau$ controls the hardness of the assignment, and the special case of $\tau=0$ corresponds to the standard deep-MIML framework. 
Interestingly, this mask can also be interpreted as an attention weight matrix, which is then multiplied with the multi-instance logit matrix $\Vb_2$. 
While there is other literature using attention for MIL~\cite{ilse2018attention}, none of the existing methods uses a robust calculation of the label embedding as the input to the attention. The proposed co-attention framework is easily interpretable since both labels and samples lie in the same representation space, with theoretical justifications we have shown in Section \ref{sec:analysis}. 
Notice that the co-attention framework in Figure \ref{fig:miml} can be trained end-to-end. 

\section{Experimental Results} 
\label{sec:exp}

\begin{table*}[t]
	\centering
	\small 
	\caption{Statistics of $4$ XMC datasets. `sample size' column includes training \& test set.  The last column includes precisions with the clean datasets, which can be thought of as the \textit{oracle} performance given an XMC dataset with aggregated labels.}
	\noindent\makebox[\textwidth]{
	\begin{tabularx}{1.2\textwidth}{l ccc ccc cc}
		\toprule
		\textbf{Dataset} 		& \textbf{\# feat.} 	& \textbf{\# label}	& \textbf{sample size} 		& \textbf{avg samples/label}	& \textbf{avg labels/sample} & \textbf{std. precision (P@1/3/5)}	\\
		\midrule 
		\textbf{EurLex-4K} 		& 5,000 	& 3,993 	& 15,539 / 3,809 		& 25.73 		& 5.31 & 82.71 / 69.42 / 58.14	\\
		\textbf{Wiki-10K} 		& 101,938 	& 30,938	& 14,146 / 6,616 		& 8.52			& 18.64	& 84.31 / 72.57 / 63.39\\
		\textbf{AmazonCat-13K} 	& 203,882 	& 13,330 	& 1,186,239 / 306,782	& 448.57 		& 5.04 & 93.03 / 79.16 / 64.52	\\
		\textbf{Wiki-325K} 		& 1,617,899	& 325,056 	& 1,778,351 / 587,084	& 17.46 		& 3.19 & 66.04 / 43.63 / 33.05	\\
		\bottomrule
	\end{tabularx} 
	}
	\label{tab:xmc-exp-stats}
\end{table*}

\begin{figure}
	\centering
		\includegraphics[width=0.5\linewidth]{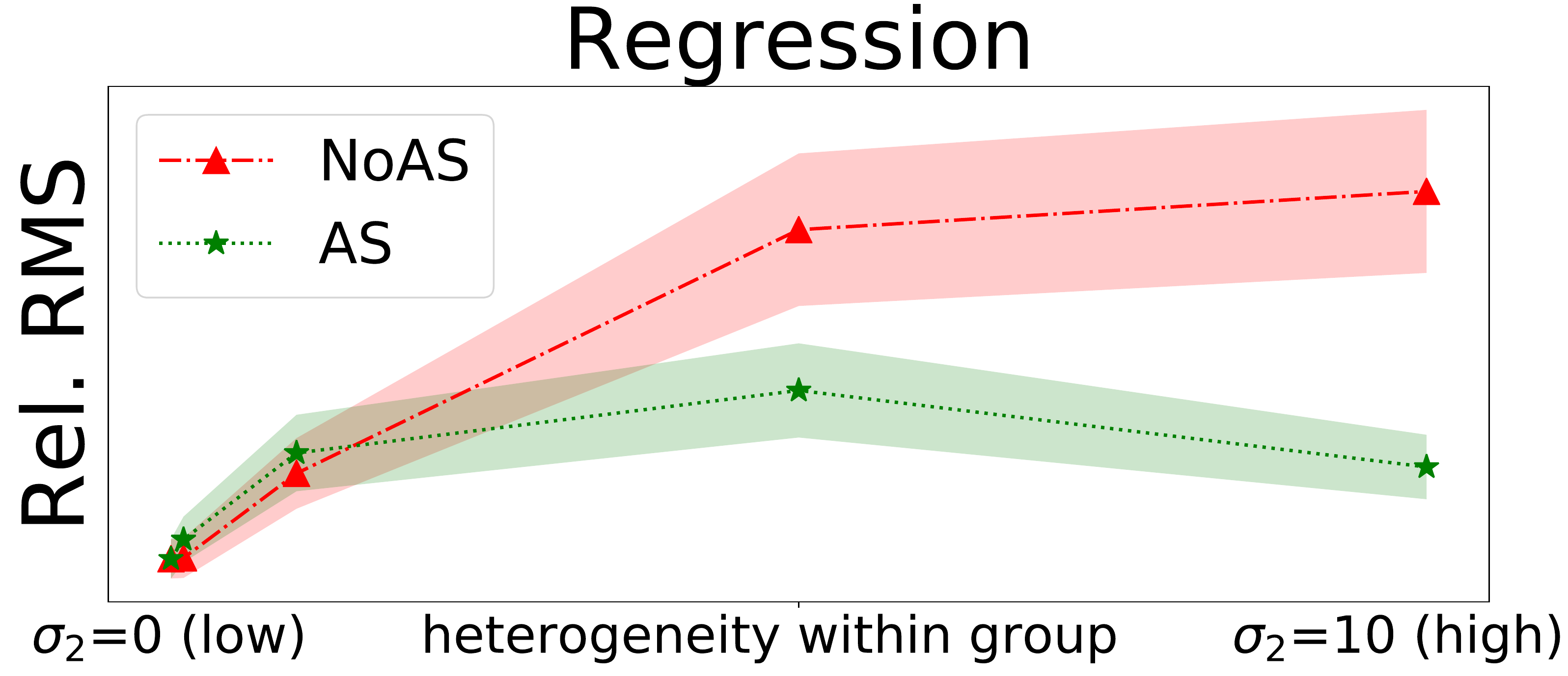}
	\caption{Regression task with aggregated outputs. The advantage of \textbf{\estimatorb}~(estimator w. label assignment) over \textbf{\estimatora}~(estimator w.o. label assignment) matches with the result in Theorem \ref{thm:1}. The $y$-axis is the root mean square (RMS) value normalized by RMS of the maximum likelihood estimator with known correspondences (lower is better). 
	}
	\label{fig:sim}
\end{figure}

In this section, we empirically verify the effectiveness of \algname~from multiple standpoints. 
First, we run simulations to verify and explain the benefit of label assignment as analyzed in Theorem \ref{thm:1}. 
Next, we run synthetic experiments on standard XMC datasets to understand  the  advantages of \algname~under multiple aggregation rules. 
Lastly, for the natural extension of \algname~in the non-extreme setting (as mentioned in Section \ref{sec:alg_extension}), we study multiple MIML tasks and show the benefit of \algname~over standard MIML solution. 
We include details of the experimental settings and more comparison results in the   Appendix.

\subsection{Simulations}
We design a toy regression task to better explain the performance of our approach from an empirical perspective. 
Our data generating process strictly follows the setting in Theorem \ref{thm:1}. We set $\sigma_1=\sigma_e=1.0$ and vary $\sigma_2$ from $0.0$ to $10.0$, which corresponds to heterogeneity within group changes from low to high.

\paragraph{Results.} In Figure \ref{fig:sim}, as the deviation within each sample group increases, \textbf{\estimatorb}~performs much better, which is due to the $\sqrt{g}$ difference in the error rate between (\ref{eqt:thm-1-1}) and the first term in (\ref{eqt:thm-1-2}). On the other hand, \textbf{\estimatorb}~may perform slightly worse than \textbf{\estimatora}~in the well-clustered setting, which is due to the second term in (\ref{eqt:thm-1-2}). 
See another toy classification task with similar observations in the Appendix.

\begin{figure*}
	\centering
	\subfigure{
		\includegraphics[width=.31\linewidth]{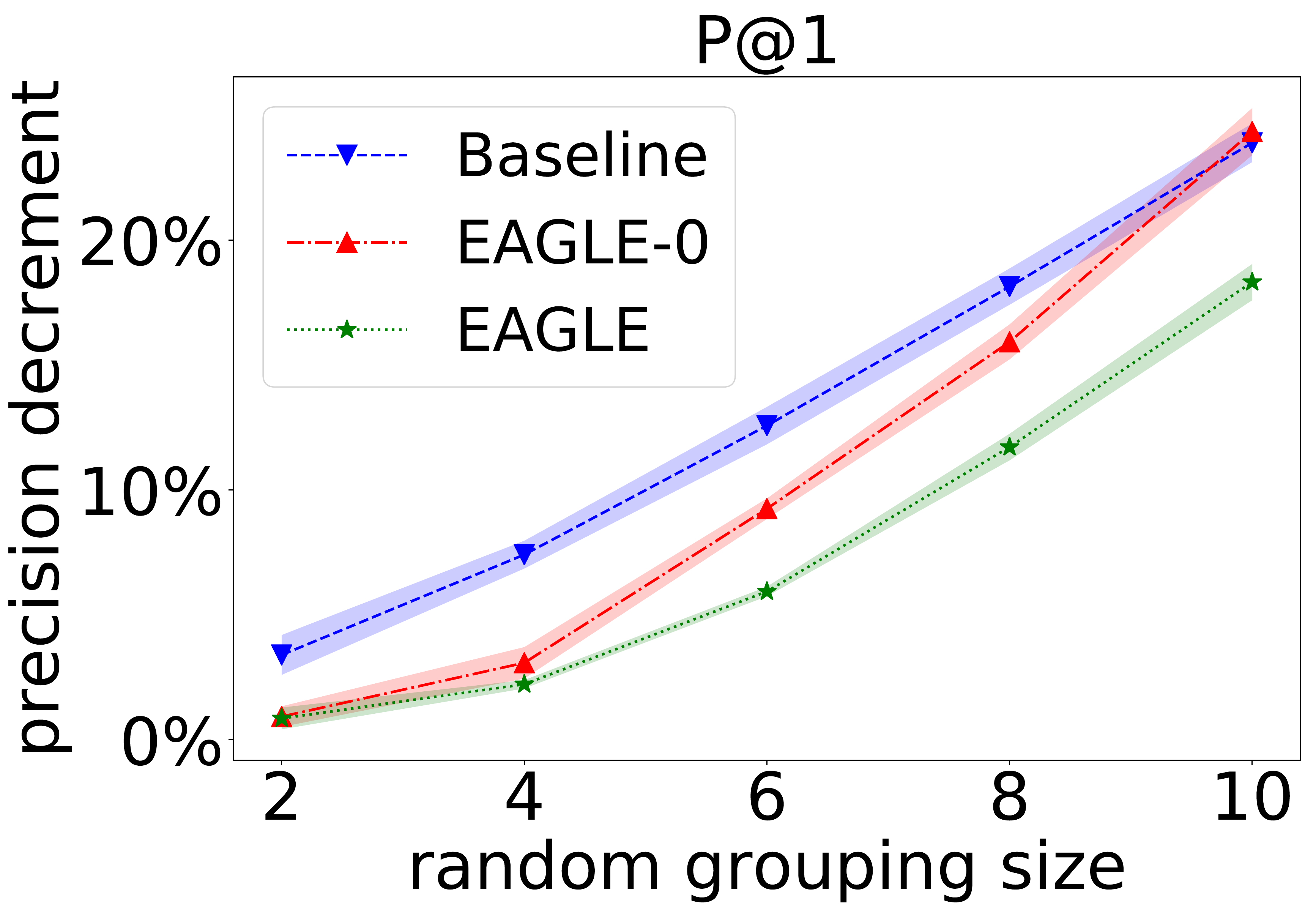}
	}
	\subfigure{
		\includegraphics[width=.31\linewidth]{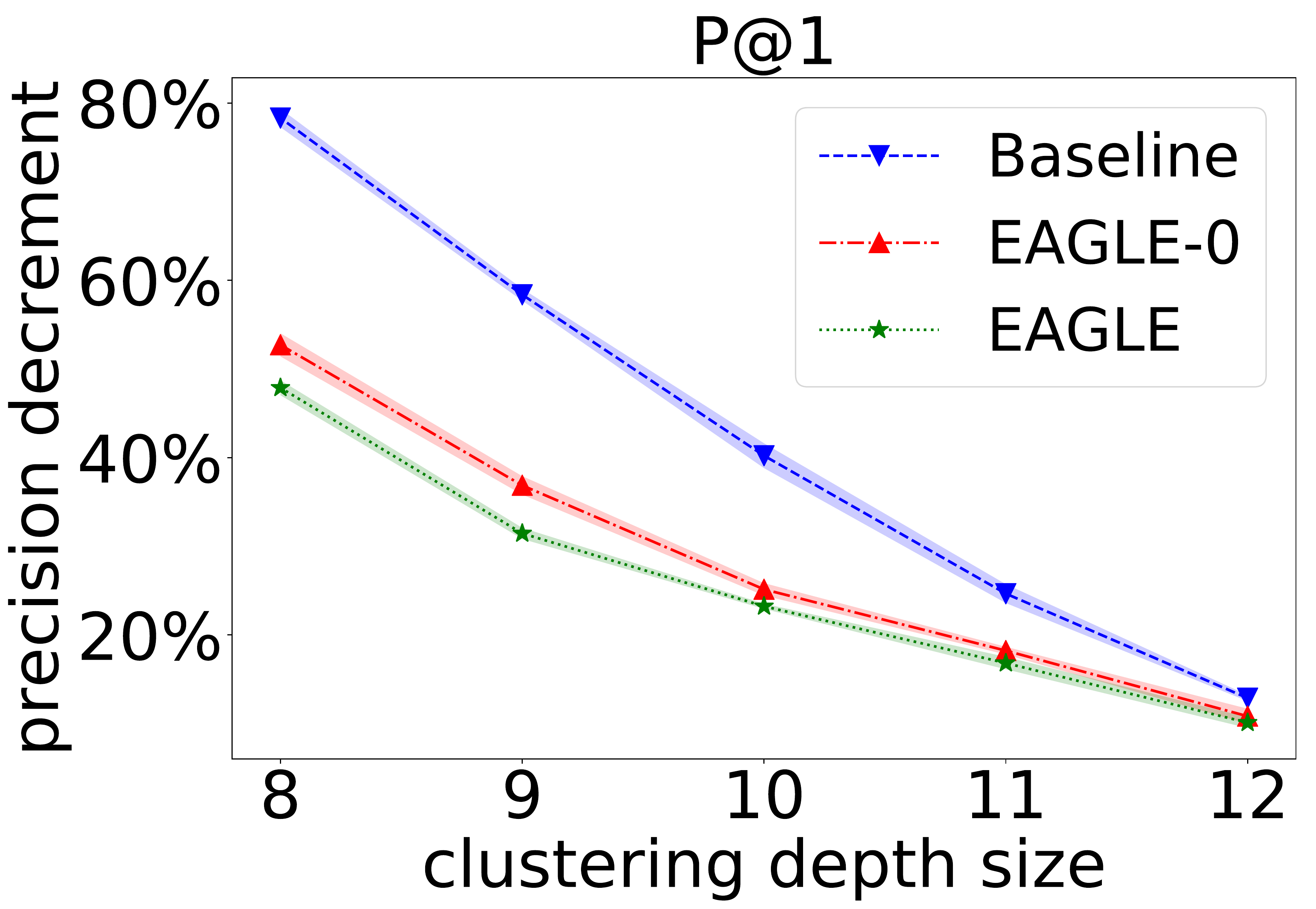}
	}
	\subfigure{
		\includegraphics[width=.31\linewidth]{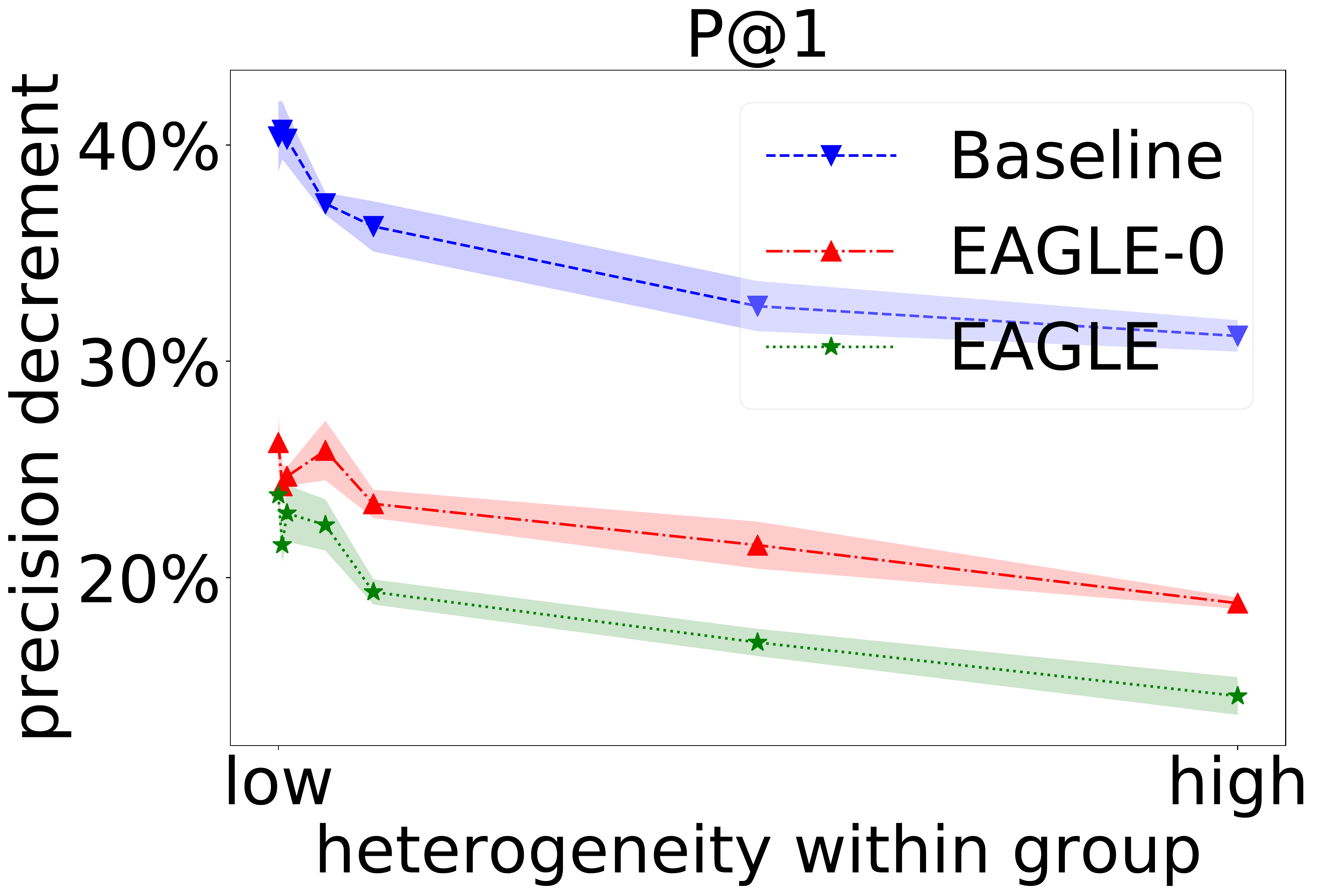}
	}
	\caption{Ablation study: comparing  Precision@1 between \textbf{Baseline} (no label assignment), \textbf{\algname-0} (\algname~without label learning) and \textbf{\algname} (\algname~with label learning). The $y$-axis calculates the percentage of precision decrement over the oracle performance (trained with known sample-label correspondences). We study different factors including \textbf{left}: group size in random grouping; \textbf{middle}: hierarchical clustering depth size; \textbf{right}: heterogeneity within group  changes from low (hierarchical clustering)  to high (random grouping). }
	\label{fig:eurlex-ablation}

\end{figure*}

\subsection{Extreme Multi-label Experiments}

\begin{table*}[]
	\centering
	\small 
	\caption{ Comparing \textbf{Baseline}, \textbf{\algname-0} (\algname~without label learning) and \textbf{\algname}  on small/mid/large-size XMC datasets with aggregated labels. `O' stands for oversized model (>5GB). \textbf{R-4/10} randomly selects $4/10$ samples in each group and observes their aggregated labels. \textbf{C} clusters samples based on hierarchical  k-means. The cluster depth is determined based on sample size ($8$ for EurLex-4k, Wiki-10k and $16$ for AmazonCat-13k and Wiki-325k). }
	\noindent\makebox[\textwidth]{
	\begin{tabularx}{1.2\textwidth}{l c | ccc|ccc|ccc|cccccc }
		\toprule
		& & \multicolumn{3}{c|}{\textbf{EurLex-4k}} & \multicolumn{3}{c|}{\textbf{Wiki-10k}} & \multicolumn{3}{c|}{\textbf{AmazonCat-13k}} & \multicolumn{3}{c}{\textbf{Wiki-325k}} \\
		& & {\scriptsize \bf  Baseline} & {\scriptsize \bf  EAGLE-0} & {\scriptsize \bf  EAGLE} & {\scriptsize \bf Baseline} & {\scriptsize \bf  EAGLE-0} & {\scriptsize \bf  EAGLE} & {\scriptsize \bf  Baseline} & {\scriptsize \bf  EAGLE-0} & {\scriptsize \bf  EAGLE}& {\scriptsize \bf  Baseline} & {\scriptsize \bf  EAGLE-0} & {\scriptsize \bf  EAGLE} \\
		\midrule 
		R-4&P1&$76.58$&$80.16$&$\mathbf{80.87}$&$73.07$&$78.28$&$\mathbf{80.38}$&$80.62$&$78.83$&$\mathbf{81.34}$&$34.09$&$\mathbf{62.45}$&$60.60$\\
		&P3&$61.36$&$63.49$&$\mathbf{65.33}$&$60.84$&$63.81$&$\mathbf{66.15}$&$65.81$&$67.45$&$\mathbf{69.97}$&$33.34$&$\mathbf{41.96}$&$40.23$\\
		&P5&$49.50$&$51.01$&$\mathbf{52.88}$&$53.31$&$54.66$&$\mathbf{57.22}$&$54.06$&$54.80$&$\mathbf{56.98}$&$26.05$&$\mathbf{31.23}$&$29.80$\\
		\midrule
		R-10&P1&$62.95$&$62.58$&$\mathbf{67.56}$&$64.87$&$65.21$&$\mathbf{68.89}$&$56.42$&$60.71$&$\mathbf{63.47}$&O&$52.43$&$\mathbf{54.81}$\\
		&P3&$47.72$&$44.07$&$\mathbf{48.03}$&$51.13$&$50.33$&$\mathbf{53.60}$&$47.74$&$49.05$&$\mathbf{51.59}$&O&$34.02$&$\mathbf{35.80}$\\
		&P5&$\mathbf{37.59}$&$33.25$&$35.91$&$42.89$&$41.22$&$\mathbf{44.68}$&$\mathbf{40.74}$&$38.06$&$39.47$&O&$24.84$&$\mathbf{26.18}$\\
		\midrule 
		C&P1&$17.94$&$39.11$&$\mathbf{43.11}$&$16.34$&$39.13$&$\mathbf{40.25}$&$\mathbf{74.65}$&$56.41$&$56.19$&O&$45.03$&$\mathbf{46.39}$\\
		&P3&$16.27$&$28.02$&$\mathbf{30.08}$&$16.08$&$30.70$&$\mathbf{31.21}$&$\mathbf{64.85}$&$48.32$&$48.07$&O&$27.98$&$\mathbf{28.96}$\\
		&P5&$14.83$&$22.46$&$\mathbf{23.78}$&$15.96$&$25.69$&$\mathbf{26.09}$&$\mathbf{53.61}$&$40.01$&$39.93$&O&$20.53$&$\mathbf{21.27}$\\
		\bottomrule
	\end{tabularx} 
	}
	\label{tab:amazoncat-wiki10}
\end{table*}

We first verify our idea on $4$ standard extreme classification tasks\footnote{http://manikvarma.org/downloads/XC/XMLRepository.html}(1 small, 2 mid-size and 1 large), whose detailed statistics are shown in Table \ref{tab:xmc-exp-stats}. For all tasks, the samples are grouped under different rules including: 
(i) random clustering: each group of samples are randomly selected;
(ii) hierarchical clustering: samples are hierarchically clustered using k-means. 
Each sample in the original XMC dataset belongs to exactly one of the groups. 
As described in Section \ref{sec:alg}, \algname~learns the label embeddings, and assigns every label in the group to one of the samples based on the embeddings. 
Then, we run Parabel and compare the final performance. 
Notice that it is possible to assign labels more cleverly, however, we focus on the quality of the label embedding learned through \algname~hence we stick to this simple assigning rule. 
We consider (i) \textbf{Baseline}:  Parabel without label assignment;
(ii) \textbf{\algname-0}: \algname~without label learning ($T=0$); and (iii) \textbf{\algname}: \algname~with label learning ($T=20$ by default). 

\paragraph{Results.} We report the performance using the standard Precision@1/3/5 metrics in Table \ref{tab:amazoncat-wiki10}. From the empirical results, we find that  \textbf{\algname}~performs better than \textbf{\algname-0} almost consistently, across all tasks and all grouping methods, and is much better than  \textbf{Baseline} where we ignore such aggregation structure. 
\textbf{Baseline} performs much better only on AmazonCat-13k with hierarchical clustering, which is because of the low heterogeneity within each cluster, as theoretically explained by our Theorem \ref{thm:1}. 
Notice that the precision on standard AmazonCat-13k achieves $93.04$, which implies that the samples are easily separated.  
Furthermore, we also provide ablation study on EurLex-4K in Figure \ref{fig:eurlex-ablation} to understand the influence of group size and clustering rule. 
We report the decrement percentage over a model trained with known correspondences. 
As a sanity check, as the size of the group gets smaller and annotation gets finer in Figure \ref{fig:eurlex-ablation}-(a) \& (b), all methods have $0\%$ decrement. More interestingly, in the other regime of more coarse annotations, (a) \& (b)  show that the benefit of \algname-0 varies when the clustering rule changes while the benefit of \algname~is consistent. The consistency also exists when we  change the heterogeneity within group by injecting noise to the feature representation when running the hierarchical clustering algorithm, as shown in Figure \ref{fig:eurlex-ablation}-(c).

\subsection{MIML Experiments}

\begin{figure*}
	\setlength\abovecaptionskip{2pt}
	\centering
	\begin{minipage}[t]{0.09\linewidth}
		\includegraphics[width=\linewidth]{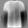}
		\caption*{T-shirt}
	\end{minipage}
	\begin{minipage}[t]{0.09\linewidth}
		\includegraphics[width=\linewidth]{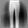}
		\caption*{Trouser}
	\end{minipage}
	\begin{minipage}[t]{0.09\linewidth}
		\includegraphics[width=\linewidth]{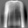}
		\caption*{Pullover}
	\end{minipage}
	\begin{minipage}[t]{0.09\linewidth}
		\includegraphics[width=\linewidth]{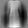}
		\caption*{Dress}
	\end{minipage}
	\begin{minipage}[t]{0.09\linewidth}
		\includegraphics[width=\linewidth]{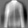}
		\caption*{Coat}
	\end{minipage}
	\begin{minipage}[t]{0.09\linewidth}
		\includegraphics[width=\linewidth]{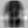}
		\caption*{Sandal}
	\end{minipage}
	\begin{minipage}[t]{0.09\linewidth}
		\includegraphics[width=\linewidth]{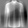}
		\caption*{Shirt}
	\end{minipage}
	\begin{minipage}[t]{0.09\linewidth}
		\includegraphics[width=\linewidth]{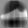}
		\caption*{Sneaker}
	\end{minipage}
	\begin{minipage}[t]{0.09\linewidth}
		\includegraphics[width=\linewidth]{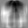}
		\caption*{Bag}
	\end{minipage}
	\begin{minipage}[t]{0.09\linewidth}
		\includegraphics[width=\linewidth]{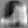}
		\caption*{Boot}
	\end{minipage}
	\caption{Visualization of learned label embeddings on Fashion-MNIST dataset.}
	\label{fig:fashion}
\end{figure*}

First, we run a set of synthetic experiments on the standard MNIST \& Fashion-MNIST image datasets, where we use the raw $784$-dimension vector as the representation. 
Each `sample' in our training set consists of $g$ random digit/clothing images and the set of corresponding labels (we set $g=4,50$). We then test the accuracy on the standard test set. 
On the other hand, we collect the standard Yelp's customer review data from the web. 
Our goal is to predict the tags of a restaurant based on the customer reviews. 
We choose $10$ labels with balanced positive samples and report the overall accuracy. 
Notice that each single review can be splitted into multiple sentences, as a result, 
we formulate it as an MIML problem similar to \cite{feng2017deep}. 
We retrieve the feature of each instance using InferSent\footnote{https://github.com/facebookresearch/InferSent}, an off-the-shelf sentence embedding method. 
We randomly collect $20$k reviews with $4$ sentences for training, $10$k reviews with single sentence for validation and $10$k reviews with single sentence for testing. 
We report the top-1 precision. 

For both the tasks, we use a two-layer feed-forward neural network as the base model, identical to the setting in \cite{feng2017deep}. 
We compare Deep-MIML with the extension of \algname-0 and \algname~for MIML, as illustrated in Figure \ref{fig:miml}. 
We first do a hyper-parameter search for Deep-MIML to find the best learning rate and the ideal epoch number. Then we fix those hyper-parameters and use them for all algorithms.

\begin{table}[]
    \centering
    \caption{Prediction accuracy on multiple MIML  tasks.    }
    \begin{tabular}{c cc c cc c cc}
        \toprule
         & \multicolumn{1}{c}{\textbf{MNIST}} &  \multicolumn{1}{c}{\textbf{Fashion}} & \multicolumn{1}{c}{\textbf{Yelp}} \\
        group size 		& $4$ / $50$ 						& $4$ / $50$ 						& $4$				\\
        \midrule 
        Deep-MIML   	& $94.70$/$33.33$ 					& $84.70$/$19.00$    				& $40.69$ 			\\
        \midrule 
        \algname-0   	& $\mathbf{94.82}$/$36.10$  		& $84.89$/$27.62$ 					& $45.82$ 			\\
        \algname  		& $\mathbf{94.82}$/$\mathbf{38.46}$	& $\mathbf{85.09}$/$\mathbf{28.65}$ & $\mathbf{46.25}$  \\
        \bottomrule
    \end{tabular} 
    \label{tab:miml}
    \vspace{-5pt}
\end{table}

\paragraph{Results.} 
The performance of \algname~in multiple MIML tasks is shown in Table \ref{tab:miml}. 
We see a $5.6\%$ absolute improvement over Deep-MIML on the Yelp dataset. 
There is consistent improvement for the image tasks as well. 
Notice that the improvement on the large group size is much significant than on the small group size. 
This is because the original Deep-MIML framework is able to handle the easy MIML tasks well, but is less effective for difficult tasks. 
Figure \ref{fig:fashion} visualizes the learned label embedding for Fashion-MNIST dataset. 
All these results corroborate the advantage of using  label embeddings at the beginning of feed-forward neural networks in MIML.

\section{Conclusions \& Discussion}
\label{sec:discussion}
In this paper, we study XMC with aggregated labels, and propose the first efficient algorithm \algname~that advances standard XMC methods in most settings. 
Our work leaves open several interesting issues to study in the future. First, while using positively labeled groups to learn label embedding, what would be the most efficient way to also learn/sample from negatively labeled groups? Second, is there a way to estimate the clustering quality and adjust the hyper-parameters accordingly? 
Moving forward, we believe the co-attention framework we proposed in Section \ref{sec:alg_extension} can help design deeper neural network architectures for MIML with better performance and interpretation.

\bibliography{ref}
\bibliographystyle{alpha}

\clearpage
\newpage

\appendix

\section{Clarifications}

\paragraph{Notations}
We could not summarize all notations in the main text due to space constraint. 
Here, we list and summarize important notations in Table \ref{tab:app-notation} for reader's reference. 

\begin{table*}[h]
	\centering
	\footnotesize 
	\caption{List of notations.}
	\noindent\makebox[\textwidth]{
	\begin{tabularx}{1.2\textwidth}{c|c||c|c||c|cccc}
		\toprule
		\multicolumn{2}{l||}{\bf Definitions related to samples}  & \multicolumn{2}{l||}{\bf Definitions related to intermediate nodes}  &   \multicolumn{2}{l}{\bf Definitions related to labels}  \\
		\midrule 
		$\Ncal$ & set of samples & $\Mcal$ & set of intermediate nodes & $\Lcal$ & set of labels \\
		$n$ & sample size & $m$ & intermediate node size & $l$ & label size \\
		$i$ & element in sample set & $j$ & element in intermediate node set & $k$ & element in label set \\
		& ( or index of a sample ) & & ( or index of an intermediate node ) & & ( or index of a label ) \\
		$\Ncal_j$ & set of samples connected to $j$ & $\Mcal_i$ & set of nodes connected to $i$ & $\Lcal_j$ & set of labels connected to $j$  \\
		 &( or set of sample indices )& &( or set of node indices )&& ( or set of label indices ) \\
		&& $\Mcal_k$ & set of nodes connected to $k$ &&\\
		&& & ( or set of node indices ) &&\\
		\midrule  
	\multicolumn{6}{l}{\bf Annotation matrix} \\
	\midrule 
	   $\Yb^1$& sample-node binary matrix  & $\Yb^2$& node-label binary matrix  & $\Yb$ & sample-label binary matrix \\
	   \midrule 
	\multicolumn{6}{l}{\bf Others related to XMC} \\ \midrule 
	$\Xb$ & data feature matrix & $\Lb$ & label embedding matrix & $\Xb_\mathcal{S}$ & submatrix of $\Xb$ with rows in $\Scal$\\
	$g$ & \# of samples in a group & $\bar{g}$ & avg. \# of samples for each group & \\
	$d$ & feature dimension & $\bar{d}$ & average sparsity of a feature \\ \midrule 
	\multicolumn{6}{l}{\bf Math related } \\ \midrule 
	$\Ab$ & general matrix & $\ab$ & general vector & $a$ & general scalar \\
	$\mathtt{nnz}(\Ab)$ & \# of non-zero entries in matrix $\Ab$ & $\mathbf{1}_d$ & all one vector with $d$ dimension & $\mathbf{1}_{d_1\times d_2}$ & all one matrix with size $d_1\times d_2$ \\
	$\otimes $& Kronecker product & $\|\Ab\|$ & spectral norm of matrix $\Ab$ & $\|\ab\|$ & $l_2$ norm of $\ab$ \\
	$\similarity{\ab_1}{\ab_2}$ & cosine similarity between $\ab_1$, $\ab_2$ & $\mathtt{Proj}(\ab)$ & $\ab/\|\ab\|$ & $\mathbf{1}\{\mathtt{e}\}$ & return $1$ if $\mathtt{e}$ is correct else $0$\\
	\midrule 
	\multicolumn{6}{l}{\bf Others} \\
	\midrule 
	$\eb_k^\star$ & ground truth label $k$'s embedding & $\eb_k/\hat{\eb}_k$ & estimate of $k$'s embedding & $\eb^t$ & estimate at $t$-th iterate \\
	\bottomrule 
	\multicolumn{6}{r}{\textit{*Theorem related definitions are not listed here. Check Theorem settings for details.} } \\
	 \bottomrule 
	\end{tabularx}
	}
	\label{tab:app-notation}
\end{table*}

\paragraph{On the efficiency of \algname}
In Section \ref{sec:alg}, we have discussed the complexity of \algname. Remind that we have $O(n\bar{d}\bar{g}\log l)$ complexity for \algname~and $O(n\bar{d}\log l), O(l\bar{d}\log l), O(n\bar{d}\log l)$ for the three steps in Parabel~\cite{prabhu2018parabel}. Notice that Parabel's complexity is analyzed under $\bar{g}=1$. As a result, if we directly apply Parabel to $\{\Xb, \Yb^1\Yb^2\}$, i.e., using the baseline approach (without label assignment), the time complexity becomes $O(n\bar{g}\bar{d}\log l), O(l\bar{d}\log l), O(n\bar{g}\bar{d}\log l)$ for the three steps, which is not better than \algname~(and slower in practice, because step 3 costs more time than step 1). 
For other XMC approaches that are less efficient, the computation complexity would also increase by a factor of $\bar{g}$. 
As a result, running the baseline approach (without label assignment) becomes less efficient compared to \algname. From the model size perspective, due to the label assignment, \algname~would not increase the model size. 
However, for the baseline approach, models that use sparse representations will see an increase in model size with a multiplicative factor roughly equals to $\bar{g}$.

\section{Additional Plots/Tables and Experimental Details}

\subsection{Simulations}

We run simulation to verify the results in Theorem \ref{thm:1}. We include a linear regression experiment and a  linear classification experiment. 
The setting for the regression experiment is as follows~\footnote{This is the equivalent to what we described in the main paper, what state it again for clarity.}: 
The input features are generated following a two-step procedure: first, the mean (center) for samples associated with each group is generated following $\Ncal(0, \sigma_1^2 \Ib_d)$; then, given the center $\xb_c$, the feature of each sample in this group is $\xb_c$ plus an additional vector that follows distribution $\Ncal(0, \sigma_2^2 \Ib_d)$.
The response of each sample, is generated following $\yb = \Bb^\star \xb_i + \epsilonb_i$, where $\Bb\in \mathbb{R}^{l\times d}$ is the parameter to be recovered, $\epsilonb \sim \Ncal(0, \sigma_e^2\Ib_d)$. 
This setting is identical to the setting we analyzed in Section \ref{sec:analysis}. 
Notice that $\sigma_2/\sigma_1$ characterizes the quality of the clustering. 

For the classification experiment, we first generate the ground truth for each label. Then, the feature of each sample is one of the ground truth label with an additional noise, and the label of the feature is the $\arg\max_i \similarity{\eb_k^\star}{\xb_i} + \epsilon_i$. 
Note that in regression setting, we used $\sigma_2/\sigma_1$ to describe how good the clustering quality is, here, we control the clustering quality by changing the ground truth labels in each group from evenly generated ($p=[1/l, \cdots, 1/l]$) to unevenly generated with probability vector $p= \frac{1}{l-1 + \exp(\sigma_2)}\cdot \mathbf{1} + \frac{\exp(\sigma_2)}{l-1 + \exp(\sigma_2)} \hb_k $ for some random label index $k\in \Lcal$, where $\hb_k$ is the one-hot vector with non-zero index at the $k$th position.

\begin{figure}[!ht]
	\centering
	\subfigure{
		\includegraphics[width=.45\linewidth]{figs/simulation_s1_1-0_regression-v4.pdf}
	}
	\subfigure{
		\includegraphics[width=.45\linewidth]{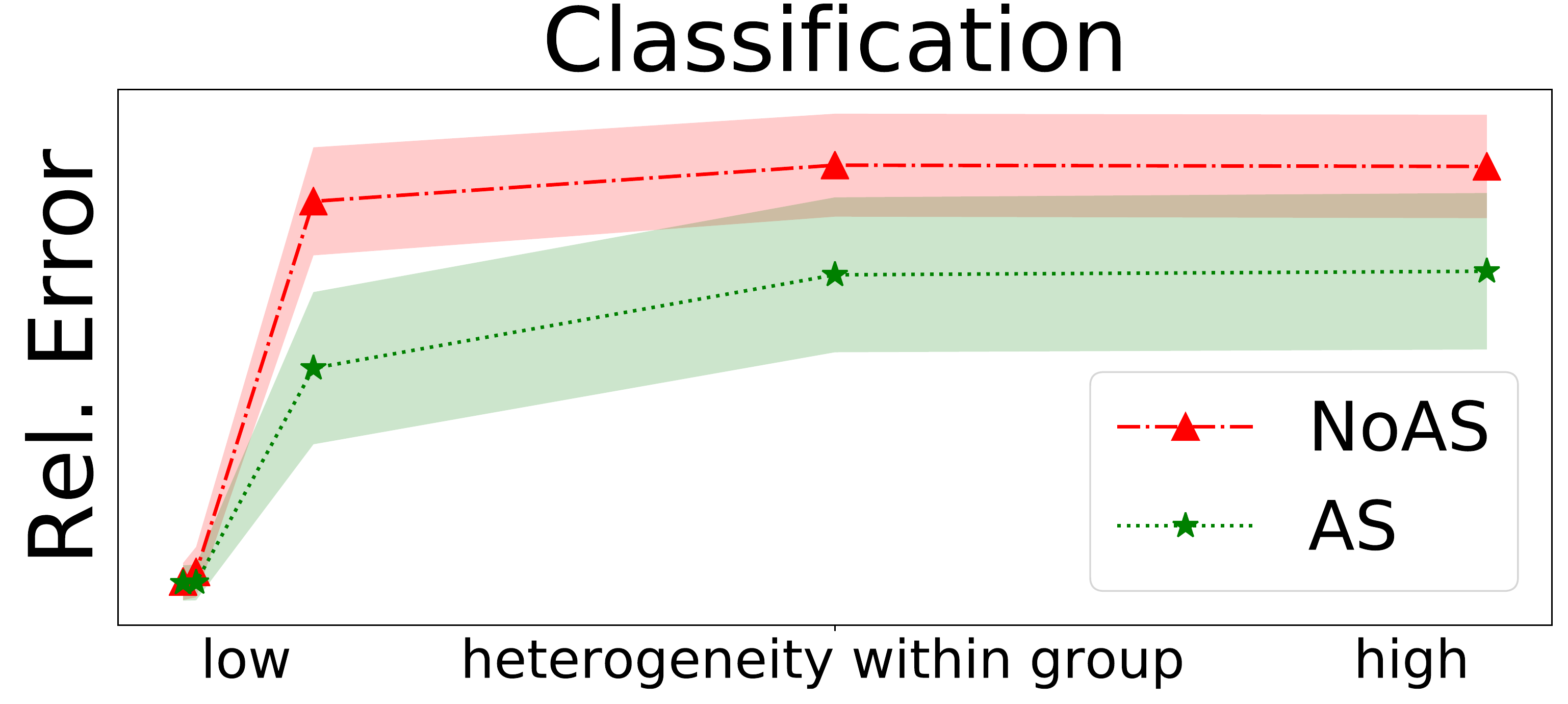}
	}
	\caption{Simulation results for a regression task (left) and a classification task (right). \textit{Rel. RMS} stands for \textit{relative root mean square error}, which calculates the ratio between RMS using the estimator over RMS using the oracle estimator with known correspondences. \textit{Rel. Error} stands for \textit{relative error}, which calculates the ratio between prediction error of the estimator over the error using the oracle estimator with known correspondences. For both plots, lower y-value corresponds to a better estimator. }
	\label{fig:app-sim}
\end{figure}

The parameters used for the experiments are as follows: $g=10$, $n=1000$, $d=10$, $l=5$, $\sigma_2 \in [0.0, 0.1, 1.0, 5.0, 10.0]$. For regression, $\sigma_1 = \sigma_e = 1.0$. For classification, $\sigma_1=0.1, \sigma_e = 0.0$. 
The results are shown in Figure \ref{fig:app-sim}. 
For both tasks, 
$\mathtt{\estimatorb}$ becomes significantly better than $\mathtt{\estimatora}$ as the heterogeneity within the group becomes higher. 

\subsection{Experiments for XMC datasets}

\paragraph{Details.} We use a default learning rate $\lambda =0.1$ and iteration number $T=20$ for all the experiments. When generating the \problemname~dataset, we aggregate label set of each sample in every group by simple list merge operation, and there may be repetitions in the list of aggregated labels. One can also use set merge that guarantees no repetitions in each aggregated group, and the results should not be significantly different. Nevertheless, this is an experimental detail we did not point out in the main text due to space constraint. For the optimal iteration number $T$, we find that increasing $T$ does not always increase the final performance. From the theoretical perspective, Theorem \ref{thm:3} only shows contraction when the previous iterate is out of the noise region. In other words, we have converge up to a noise ball. As a result, the quality of the learned embedding may get slightly worse (up to the noise level) as $T$ increase. For Wiki-10k experiments with hierarchical clustering, the reported results use $T=5$. 

\paragraph{Additional plots. } We include here additional tables / plots for the XMC experiments. 
Table \ref{tab:xmc-eurlex} presents the detailed values we got on the EurLex-4k dataset, with standard deviation calculated on $5$ random runs. 
Figure \ref{fig:app-1} and \ref{fig:app-2} show the comparison under random grouping and hierarchical clustering setting,  respectively. 
Figure \ref{fig:app-3} shows how the improvement changes when the heterogeneity within group gradually changes from low  to high. Low heterogeneity corresponds to hierarchical clustering while high heterogeneity corresponds to random grouping. For all three figures, we include the Precision @1/3/5 metrics (whereas in the main text, only Precision @1 is shown). 
Notice that cluster depth equals $d$ corresponds to $2^d$ groups in total. 
We can see consistent improvement in all plots, while we can also see that Precision @5 performs relatively worse than Precision @1. 
Part of the reason is because our algorithm does not take into account the fact that each sample may have multiple labels with different importance. 
Improving the quality of improvement for less important labels would be an interesting problem to study in the future. 

\begin{figure*}[!ht]
	\centering
	\includegraphics[width =\columnwidth ]{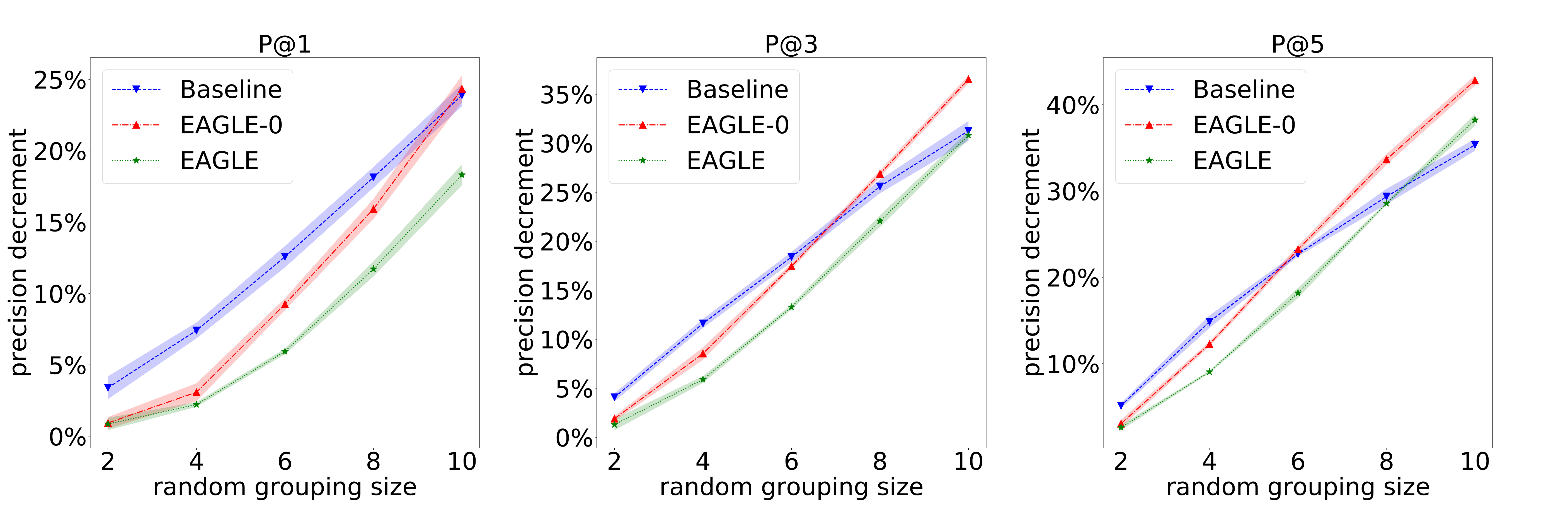}
	\caption{The effect of random grouping size to the performance. }
	\label{fig:app-1}
\end{figure*}

\begin{figure*}[!ht]
	\centering
	\includegraphics[width =\columnwidth ]{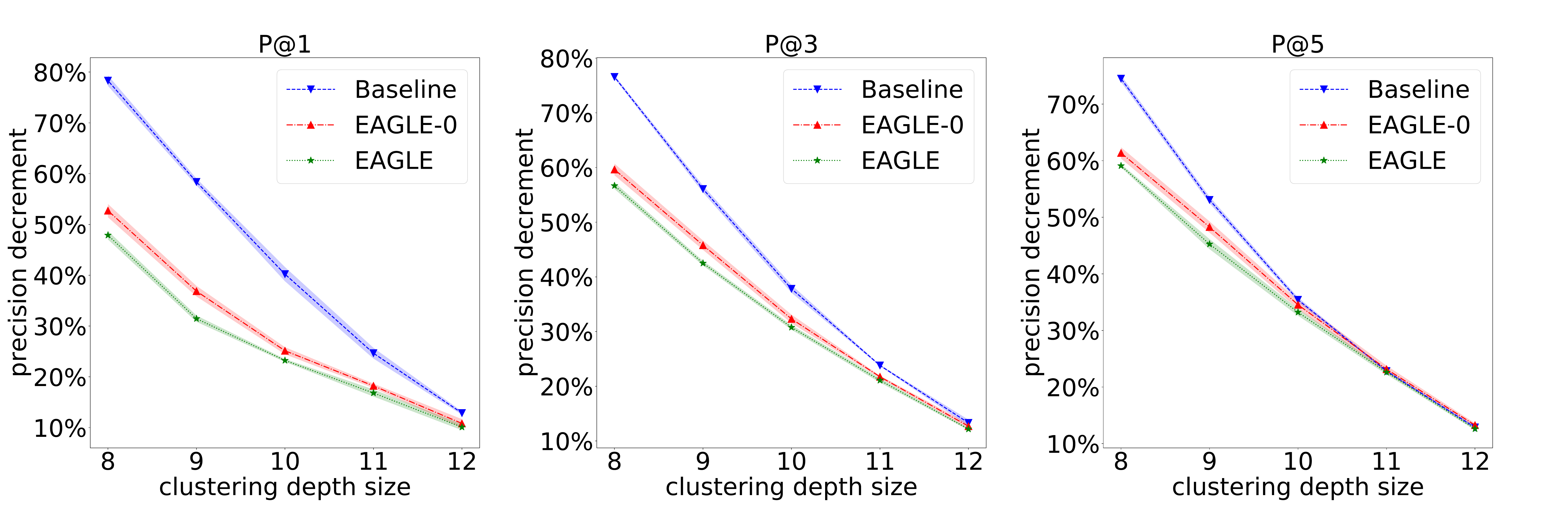}
	\caption{The effect of cluster depth to the performance. }
	\label{fig:app-2}
\end{figure*}

\begin{figure*}[!ht]
	\centering
	\includegraphics[width =\columnwidth ]{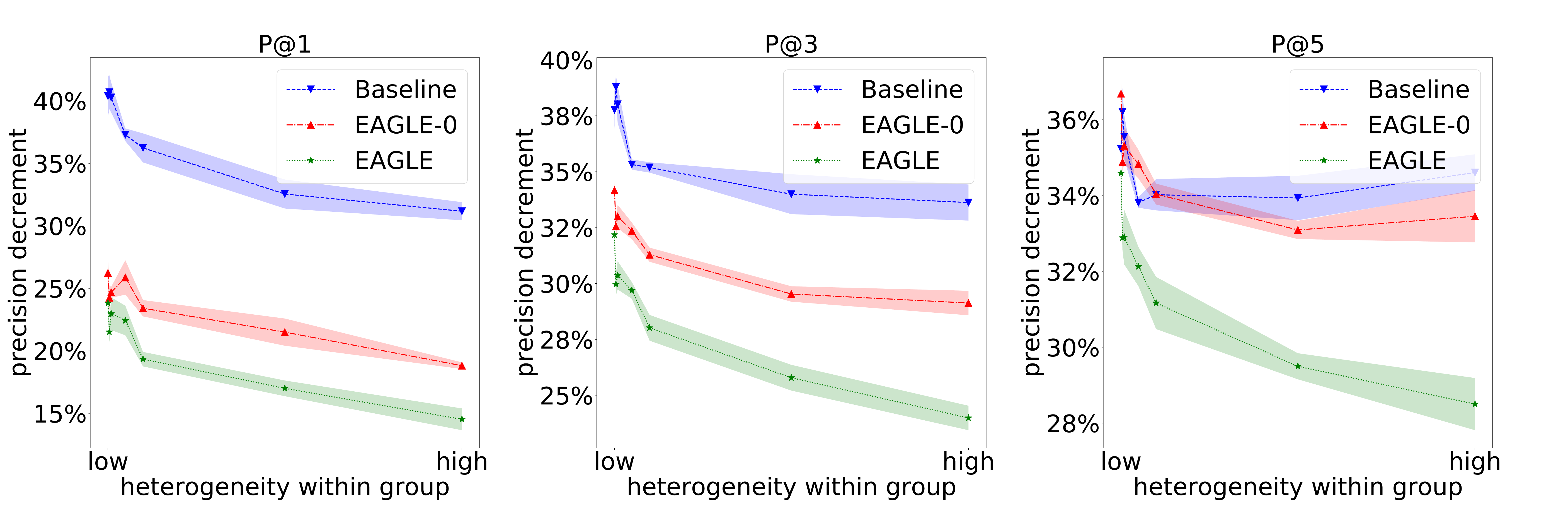}
	\caption{The effect of  within group heterogeneity to the performance. }
	\label{fig:app-3}
\end{figure*}

\begin{table*}[]
	\centering
	\caption{EurLex-4k detailed performance results. We list the standard deviation for each precision score, calculated over the result of $5$ random seeds.  }
	\begin{tabular}{c cccccc}
		\toprule
		  & & {\bf Oracle} & {\bf Baseline} & {\bf \algname-0 } & {\bf \algname} &  	\\
		\midrule 
		\multicolumn{5}{l}{\bf random grouping }\\ \midrule 
		{\bf group size} \\\midrule 
		2 	& P1 &$82.71\pm0.25$ 	& $79.90\pm0.33$ 	& $81.94\pm0.17$ 	& $82.00\pm0.18$ 	\\
		& P3 & $69.42\pm0.07$ 	& $66.59\pm0.14$ 	& $68.10\pm0.09$ 	& $68.52\pm0.20$ 	\\
		& P5 & $58.14\pm0.06$ 	& $55.16\pm0.09$ 	& $56.37\pm0.13$ 	& $56.63\pm0.09$ 	\\ \midrule 
		
		4 	& P1 & - 				& $76.58\pm0.23$ 	& $80.16\pm0.26$ 	& $80.87\pm0.07$	\\
		& P3 & - 				& $61.36\pm0.16$ 	& $63.49\pm0.22$ 	& $65.33\pm0.12$	\\
		& P5 & - 				& $49.50\pm0.21$ 	& $51.01\pm0.08$ 	& $52.88\pm0.04$ 	\\ \midrule
		
		6 	& P1 & - 				& $72.31\pm0.31$ 	& $75.06\pm0.17$ 	& $77.80\pm0.08$ 	\\
		& P3 & - 				& $56.66\pm0.18$ 	& $57.31\pm0.10$ 	& $60.20\pm0.09$ 	\\
		& P5 & - 				& $44.94\pm0.08$ 	& $44.63\pm0.12$ 	& $47.57\pm0.17$ 	\\ \midrule
		
		8 	& P1 & - 				& $67.71\pm0.30$ 	& $69.54\pm0.29$ 	& $73.02\pm0.22$	\\
		& P3 & - 				& $51.66\pm0.24$ 	& $50.76\pm0.12$ 	& $54.11\pm0.21$ 	\\
		& P5 & - 				& $41.07\pm0.26$ 	& $38.57\pm0.18$ 	& $41.53\pm0.05$ 	\\ \midrule
		
		10 	& P1 & - 				& $62.95\pm0.32$ 	& $62.58\pm0.39$ 	& $67.56\pm0.30$ 	\\
		& P3 & - 				& $47.72\pm0.35$ 	& $44.07\pm0.14$ 	& $48.03\pm0.15$ 	\\
		& P5 & - 				& $37.59\pm0.20$ 	& $33.25\pm0.17$ 	& $35.91\pm0.20$ 	\\
		\midrule 
		\multicolumn{5}{l}{\bf hierarchical clustering }\\ \midrule 
		{\bf cluster depth} \\ \midrule 
		8 	& P1 & - 				& $17.94\pm0.43$ 	& $39.11\pm0.54$ 	& $43.11\pm0.35$ 	\\
		 	& P3 & - 				& $16.27\pm0.10$ 	& $28.02\pm0.37$ 	& $30.08\pm0.21$	\\
		& P5 & - 				& $14.83\pm0.17$ 	& $22.46\pm0.29$ 	& $23.78\pm0.09$ 	\\ \midrule
		
		9 	& P1 & - 				& $34.43\pm0.28$ 	& $52.22\pm0.43$ 	& $56.70\pm0.26$	\\
		 	& P3 & - 				& $30.50\pm0.22$ 	& $37.63\pm0.30$ 	& $39.92\pm0.15$ 	\\
		& P5 & - 				& $27.29\pm0.17$ 	& $30.07\pm0.29$ 	& $31.83\pm0.26$ 	\\ \midrule
		
		10 	& P1 & - 				& $49.45\pm0.58$	& $61.92\pm0.29$ 	& $63.50\pm0.12$ 	\\
		 	& P3 & - 				& $43.17\pm0.25$	& $46.99\pm0.27$ 	& $48.06\pm0.15$	\\
		& P5 & - 				& $37.55\pm0.12$	& $38.06\pm0.22$ 	& $38.83\pm0.18$ 	\\ \midrule
		
		11 	& P1 & - 				& $62.32\pm0.44$ 	& $67.64\pm0.18$ 	& $68.80\pm0.29$ 	\\
		 	& P3 & - 				& $52.90\pm0.05$ 	& $54.34\pm0.08$ 	& $54.79\pm0.14$	\\
		& P5 & - 				& $44.87\pm0.05$ 	& $44.68\pm0.18$ 	& $45.00\pm0.11$ 	\\ \midrule
		
		12 	& P1 & - 				& $72.07\pm0.14$ 	& $73.76\pm0.35$ 	& $74.36\pm0.25$	\\
		 	& P3 & - 				& $60.19\pm0.20$ 	& $60.61\pm0.16$ 	& $60.98\pm0.07$	\\
		& P5 & - 				& $50.66\pm0.07$ 	& $50.43\pm0.10$ 	& $50.81\pm0.07$	\\
		\bottomrule
	\end{tabular} 
	\label{tab:xmc-eurlex}
\end{table*}

\subsection{MIML experiments}
\begin{table*}[]
	\centering
	\caption{Hyper-parameter search for Deep-MIML on the Yelp dataset. 
	We search over batch size (bs) in $\{32, \cdots, 512\}$, and learning rate (lr) in $\{ 1e-4, \cdots, 0.2\}$. We select bs$=64$, lr=$1e-4$ for our experiments.  }
	\begin{tabular}{c cccccccc}
		\toprule
		\textbf{Acc.} &   \textbf{$lr=0.1$} & \textbf{$lr=0.05$} & \textbf{$lr=0.01$} & \textbf{$lr=0.005$} &  \textbf{$lr=0.001$} & \textbf{$lr=0.0005$} &  \textbf{$lr=0.0001$} \\
		\midrule 
		$bs=32$ & $10.96\pm 0.41$ & $10.82\pm 0.49$ & $11.53\pm 0.85$ & $14.44 \pm 2.69$ & $27.62\pm 3.55$ & $37.65\pm 2.85$ & $40.70\pm 1.02$ \\
		$bs=64$ & $11.20\pm 0.34$ & $10.95\pm 0.62$ & $12.08\pm 1.89$ & $15.87\pm 2.92$ & $27.06\pm 4.42$ & $38.70 \pm 2.77$ & $40.69\pm 1.43$ \\
		$bs=128$ & $11.30\pm 0.28$ & $11.17\pm 0.70$ & $11.33\pm 0.80$ & $11.81\pm 1.41$ & $24.86\pm 5.94$ & $31.48\pm 3.52$ & $40.71\pm 0.65$\\
		$bs=256$ & $11.60\pm 0.41$ & $11.44\pm 0.82$ & $11.46\pm 0.63$ & $17.92\pm 3.38$ & $20.43\pm 5.50$ & $32.97\pm 2.65$ & $39.44\pm 1.52$ \\
		$bs=512$ & $11.28\pm 0.33$ & $11.76\pm 0.78$ & $11.19\pm 1.13$ & $15.40\pm 2.96$ & $25.20\pm 4.87$ & $32.29\pm 2.74$ & $38.36\pm 0.88$\\
		\bottomrule
	\end{tabular} 
	\label{tab:miml-hpo}
\end{table*}

\paragraph{Yelp dataset. } \footnote{https://www.yelp.com/dataset}
The labels we used for the Yelp review dataset are: 
' Beauty \& Spas', ' Burgers', ' Pizza', 'Food', 
' Coffee \& Tea', ' Mexican', 
' Arts \& Entertainment', ' Italian', ' Seafood', ' Desserts', ' Japanese'. These labels have balanced number of samples, which makes precision the correct metric to use.  The original sentence embedding has dimension $4096$. 
For the convenience of the experiment, we reduce the dimension to $512$ for all the embeddings through random projection. 

\paragraph{MNIST/Fashion-MNIST dataset} We use the $784$-dimension raw input as the feature, and subtract the average over all training samples. 
We generate the MIML dataset following the same principle as the \problemname~experiments. Notice that for group size being $50$, we   observe a list of possibly repetitive labels. Doing the set merge operation when aggregating the labels does not make a lot of sense here, since with very high probability, each group will be labeled by all $10$ labels, and there is nothing to be learned from. 
We also visualize the learned label embeddings for the Fashion-MNIST experiment (with group size $4$) in Figure \ref{fig:app-fashion}. 

\begin{figure*}
	\setlength\abovecaptionskip{2pt}
	\centering
	\begin{minipage}[t]{0.09\linewidth}
		\includegraphics[width=\linewidth]{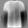}
		\caption*{T-shirt}
	\end{minipage}
	\begin{minipage}[t]{0.09\linewidth}
		\includegraphics[width=\linewidth]{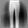}
		\caption*{Trouser}
	\end{minipage}
	\begin{minipage}[t]{0.09\linewidth}
		\includegraphics[width=\linewidth]{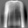}
		\caption*{Pullover}
	\end{minipage}
	\begin{minipage}[t]{0.09\linewidth}
		\includegraphics[width=\linewidth]{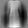}
		\caption*{Dress}
	\end{minipage}
	\begin{minipage}[t]{0.09\linewidth}
		\includegraphics[width=\linewidth]{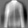}
		\caption*{Coat}
	\end{minipage}
	\begin{minipage}[t]{0.09\linewidth}
		\includegraphics[width=\linewidth]{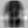}
		\caption*{Sandal}
	\end{minipage}
	\begin{minipage}[t]{0.09\linewidth}
		\includegraphics[width=\linewidth]{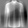}
		\caption*{Shirt}
	\end{minipage}
	\begin{minipage}[t]{0.09\linewidth}
		\includegraphics[width=\linewidth]{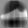}
		\caption*{Sneaker}
	\end{minipage}
	\begin{minipage}[t]{0.09\linewidth}
		\includegraphics[width=\linewidth]{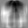}
		\caption*{Bag}
	\end{minipage}
	\begin{minipage}[t]{0.09\linewidth}
		\includegraphics[width=\linewidth]{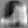}
		\caption*{Boot}
	\end{minipage}
	\caption{Visualization of learned label embedding on Fashion-MNIST dataset.}
	\label{fig:app-fashion}
\end{figure*}

\section{Proofs}

\begin{proof}[Proof of Theorem \ref{thm:1}]
	We have $n$ samples splitted into $m=n/g$ groups, each with $g$ samples. In the regression setting, both $\Ncal_j$ and $\Lcal_j$ refer to the same set of samples, so we use $\Ncal_j$ for clarity. We follow the close form solution of linear regression and the $\mathbf{\mathtt{\estimatora}}$ estimator can be written as:
	\begin{align*}
		\hat{\Bb}_{\mathtt{\estimatora}} = & \mathtt{LR}\left( \cup_{j\in \Mcal} \left\{ \left( \mathbf{1}_g^\top \Xb_{\Ncal_j}, \mathbf{1}_g^\top \Zb_{\Ncal_j}  \right) \right\} \right) \\
		= & \left( \begin{bmatrix}
		{\Xb_{\Ncal_1}}^\top \mathbf{1}_g & \cdots  & {\Xb_{\Ncal_m}}^\top \mathbf{1}_g 
		\end{bmatrix} 
		\begin{bmatrix}
		\mathbf{1}_g^\top \Xb_{\Ncal_{1}}\\ 
		\hdots \\ 
		\mathbf{1}_g^\top \Xb_{\Ncal_{m}}
		\end{bmatrix}
		 \right)^{-1}
		 \begin{bmatrix}
		 {\Xb_{\Ncal_1}}^\top \mathbf{1}_g & \cdots  & {\Xb_{\Ncal_m}}^\top \mathbf{1}_g 
		 \end{bmatrix} 
		 \begin{bmatrix}
		 \mathbf{1}_g^\top \Yb_{\Ncal_1}\\ 
		 \hdots \\ 
		 \mathbf{1}_g^\top \Yb_{\Ncal_1}
		 \end{bmatrix}.
    \end{align*}
    Re-organizing the above equation we get: 
    \begin{align*}
		 \hat{\Bb}_{\mathtt{\estimatora}} = & \left(  \sum_{j\in \Mcal} {\Xb_{\Ncal_j}}^\top \mathbf{1}_g\mathbf{1}_g^\top \Xb_{\Ncal_j} \right)^{-1} \sum_{j\in \Mcal }{\Xb_{\Ncal_j}}^\top \mathbf{1}_g \mathbf{1}_g^\top \Yb_{\Ncal_j} \\
		 = & \left(  \sum_{j\in \Mcal} {\Xb_{\Ncal_j}}^\top \mathbf{1}_g\mathbf{1}_g^\top \Xb_{\Ncal_j} \right)^{-1} \sum_{j\in \Mcal}{\Xb_{\Ncal_j}}^\top \mathbf{1}_g \mathbf{1}_g^\top \Pib^j \left(\Xb_{\Ncal_j} \Bb^\star + \Eb^j \right) \\
		 = & \Bb^\star + \left(  \sum_{j\in \Mcal} {\Xb_{\Ncal_j}}^\top \mathbf{1}_g\mathbf{1}_g^\top \Xb_{\Ncal_j} \right)^{-1} \sum_{j\in \Mcal} {\Xb_{\Ncal_j}}^\top \mathbf{1}_g\mathbf{1}_g^\top \Eb^j, 
	\end{align*}
	where the last equation uses the fact that $\mathbf{1}_g^\top \Pib^j = \mathbf{1}_g$ for any permutation matrix $\Pib^j$. 
	Therefore,
	\begin{align}
	\left\| \hat{\Bb}_{\mathtt{\estimatora}} - \Bb^\star \right\|^2 = & \left\|  \left(\sum_{j\in \Mcal} {\Xb_{\Ncal_j}}^\top \mathbf{1}_g\mathbf{1}_g^\top \Eb^i \right)^\top \left(  \sum_{j\in \Mcal} {\Xb_{\Ncal_j}}^\top \mathbf{1}_g\mathbf{1}_g^\top \Xb_{\Ncal_j} \right)^{-2} \sum_{j\in \Mcal} {\Xb_{\Ncal_j}}^\top \mathbf{1}_g\mathbf{1}_g^\top \Eb^j  \right\|. \label{eqt:app-thm1-1}
	\end{align}
	Based on how $\Xb_{\Ncal_j}$ is generated, we know that $\bar{\Xb}_j := \frac{1}{g}\mathbf{1}_g \Xb_{\Ncal_j} \sim \mathcal{N}(0, (\sigma_1^2 + \sigma_2^2/g) \Ib_d)$. Moreover, $\{ \bar{\Xb}_j \}_{j\in \Mcal}$ are independent and identically distributed. 
	Based on concentration property, we know that the middle term in (\ref{eqt:app-thm1-1}) is lower and upper bounded by  $$\left[n^{-2}g^{-2}(\sigma_1^2 + \sigma_2^2/g)^{-2}\left(1-O\left(\sqrt{\frac{d}{n}}\right)\right), n^{-2}g^{-2}(\sigma_1^2 + \sigma_2^2/g)^{-2}\left(1+O\left(\sqrt{\frac{d}{n}}\right)\right) \right]. $$ 
	Therefore,
	\begin{align*}
		\left\| \hat{\Bb}_{\mathtt{\estimatora}} - \Bb^\star \right\|^2 \le &  n^{-2}g^{-2}(\sigma_1^2 + \sigma_2^2/g)^{-2}\left(1+O\left(\sqrt{\frac{d}{n}}\right)\right)  \times \left(  n g^3 \left( \sigma_1^2 + \frac{\sigma_2^2}{g} \right) \frac{d\sigma_e^2}{g} \right) \left(1 +O\left( \sqrt{\frac{d}{n}} \right) \right) \\
		= & \frac{d\sigma_e^2}{n\left(\sigma_1^2 + \frac{\sigma_2^2}{g}\right)}\left(1+ O\left(\sqrt{\frac{d}{n}}\right) \right).
	\end{align*}
	
	Next, let us analyze $\hat{\Bb}_{\mathtt{\estimatorb}}$. 
	We use $\Ab^j$ to denote a binary square matrix with size $g$ such that each row has a unique non-zero entry ($\Ab^j$ does not need to be a permutation matrix, each column may include multiple non-zero entries or none). This $\Ab^j$ matrix describes the assignment happening within each group $j\in \Mcal$. 
	For convenience, let us assume $\Pib^i = \Ib_g$. 
	By definition, we have:
	\begin{align*}
		\hat{\Bb}_{\mathtt{\estimatorb}} = & \left( \begin{bmatrix}
		{\Xb_{\Ncal_1}}^\top {\Ab^1}^\top & \cdots  & {\Xb_{\Ncal_m}}^\top {\Ab^m}^\top 
		\end{bmatrix} 
		\begin{bmatrix}
		{\Ab^1} \Xb_{\Ncal_1}\\ 
		\hdots \\ 
		\Ab^m \Xb_{\Ncal_m}
		\end{bmatrix}
		\right)^{-1}
		\begin{bmatrix}
		{\Xb_{\Ncal_1}}^\top {\Ab^1}^\top & \cdots  & {\Xb_{\Ncal_m}}^\top {\Ab^m}^\top 
		\end{bmatrix} 
		\begin{bmatrix}
		 \Yb_{\Ncal_1}\\ 
		\hdots \\ 
		 \Yb_{\Ncal_m}
		\end{bmatrix}.
	\end{align*}
	Re-organizing the expression, we have:
	\begin{align*}
		\hat{\Bb}_{\mathtt{\estimatorb}} = & \left(  \sum_{j\in \Mcal} {\Xb_{\Ncal_j}}^\top {\Ab^j}^\top{\Ab^j} \Xb_{\Ncal_j} \right)^{-1} \sum_{j\in \Mcal}{\Xb_{\Ncal_j}}^\top {\Ab^j}^\top  (\Xb_{\Ncal_j} \Bb^\star + \Eb^j) \\
		= & \Bb^\star + \left( \underbrace{ \sum_{j\in \Mcal} {\Xb_{\Ncal_j}}^\top   \Xb_{\Ncal_j} }_{\mathcal{T}_{0} }+ \underbrace{\sum_{j\in\Mcal} {\Xb_{\Ncal_j}}^\top \Fb^j \Xb_{\Ncal_j}^i  }_{ \mathcal{T}_{1} } \right)^{-1} 
		\left(  \underbrace{
		\sum_{j\in \Mcal}{\Xb_{\Ncal_j}}^\top {\Fb^j}^\top  (\Xb_{\Ncal_j} \Bb^\star + \Eb^j) }_{\mathcal{T}_2} + \underbrace{ \sum_{j\in \Mcal}{\Xb_{\Ncal_j}}^\top   \Eb^i }_{\mathcal{T}_3}
		\right).
	\end{align*}
	Here, $\Fb^j$ is the residual matrix that quantifies the accuracy of the assignment $\Ab^j$, i.e., $\Fb^j = \Ab^j - \Ib_g$. The number of non-zero rows/columns in $\Fb^j$ is the number of incorrect assignment in group $\Ncal_j$. 
	For convenience of the analysis, we separate the expression into terms $\Tcal_0,\Tcal_1,\Tcal_2, \Tcal_3$. With these notations, we have  
	\begin{align}
		\left\| \hat{\Bb}_{\mathtt{\estimatorb}} - \Bb^\star \right\|^2 = 2 \left( \left\| \mathcal{T}_3^\top \mathcal{T}_3 \right\| + \left\| \mathcal{T}_2^\top \mathcal{T}_2 \right\| \right) \left\| {\left(\Tcal_0 + \Tcal_1\right)}^{-1}\right\|^2. \label{eqt:app-thm1-2}
	\end{align}
	Where $\Tcal_0, \Tcal_3$ are relatively easy to be controlled. We first analyze $\Tcal_0, \Tcal_3$ here.   
	Notice that different from the above analysis for $\hat{\Bb}_{\mathtt{\estimatora}}$, the random vectors in $\cup_{j\in \Mcal}\cup_{i\in\Ncal_j} \left\{ \xb_i  \right\}$ are not independent. However, we notice that the set of all vectors can be splitted into $g$ groups, where the vectors within each group are i.i.d. generated. 
	Therefore, with $n\ge O(gd\log^2 d)$, we still have 
	$\left\| \mathcal{T}_3^\top \mathcal{T}_3 \right\| \le O\left(  nd \sigma_x^2  \sigma_e^2 \right) $, and $\sigma_{\min}\left(\Tcal_0\right) = \sigma_{\max}\left(\Tcal_0\right) = O\left( n \sigma_x^2   \right)  $. 
	Plug in the result of $\Tcal_0, \Tcal_3$ into the previous expression (\ref{eqt:app-thm1-2}), we have: 
	\begin{align*}
		\left\| \hat{\Bb}_{\mathtt{\estimatorb}} - \Bb^\star \right\|_2^2  \le O\left( 
		\frac{
				n d \sigma_x^2   \sigma_e^2 + \left\| \mathcal{T}_2^\top \mathcal{T}_2 \right\| 
			}
		{
			\left( n \sigma_x^2 + \sigma_{\min}\left(\mathcal{T}_{1}\right)\right)^2
		}
		\right). 
	\end{align*}
	We next control $\Tcal_{1}$ and $\Tcal_{2}$. Let $\rtil$ be the total number of samples incorrectly assigned by $\hat{\Bb}_{\mathtt{\estimatora}}$, and we focus on  $l=1$ (Notice that with $l=1$, the residual for a fixed predictor follows Gaussian distribution, which is easier to describe. For larger $l$, the residual for a fixed predictor would follow $\chi^2$ distribution. The dependency on $l$ is not the focus of our analysis here. As a result, we stick to this simple setting. ).  
	Similar to the analysis for $\Tcal_0$, where we separate the vectors into groups and bound each group, we apply Lemma 5 in  \cite{shen2019iterative} to get $\sigma_{\min} (\Tcal_{1})=\Theta(\rtil  \sigma_x^2 )$,  $\sigma_{\max}(\Tcal_2) = O(\rtil ( \sigma_x^2 + \sigma_e^2))$. As a result, bounding $\rtil$ is the key to controlling both $\Tcal_1$ and $\Tcal_2$. 
	
	Notice that for any previous estimator $\hat{\Bb}_{\mathtt{\estimatora}}$ with upper bound $\Rcal_1$ ($\Rcal_1 := \| \hat{\Bb}_{\mathtt{\estimatora}} - \Bb^\star \|$), we know the residual for correct correspondences has variance at most $\sigma_e^2 + \Rcal_1^2 \sigma_x^2 $. 
	On the other hand, for incorrect correspondences, the variance is $\sigma_e^2 + 2\sigma_x^2 $. 
	As a result, using the result of Lemma 6 in \cite{shen2019iterative} for each group of vectors, we know that $\rtil \le c \sqrt{\frac{\sigma_e^2 + \Rcal_1^2 \sigma_x^2 }{\sigma_e^2 + 2 \sigma_x^2 }} \frac{n }{g} \cdot  g$. 
	
	As a result, the denominator is not dominated by $\sigma_{\min}(\Tcal_1)$. Therefore, 
	\begin{align*}
	\left\| \hat{\Bb}_{\mathtt{\estimatorb}} - \Bb^\star \right\|_2 \le & O\left( 
	\sqrt{
	\frac{
		n d \sigma_x^2   \sigma_e^2 + \left\| \mathcal{T}_2^\top \mathcal{T}_2 \right\| 
	}
	{
		\left( n \sigma_x^2 + \sigma_{\min}\left(\mathcal{T}_{1}\right)\right)^2
	}
	}
	\right) \\ 
	= & O\left( 
	\frac{ 
		\sqrt{
		n d \sigma_x^2   \sigma_e^2 + \left\| \mathcal{T}_2^\top \mathcal{T}_2 \right\|
		} 
	}
	{
		 n \sigma_x^2  
	}
	\right) \\
	= & O\left( 
	\frac{ 
		\sqrt{
			n d \sigma_x^2   \sigma_e^2} + n \sqrt{\sigma_e^2 + \Rcal_1^2 \sigma_x^2} \sqrt{\sigma_x^2 + \sigma_e^2}
	}
	{
		n \sigma_x^2  
	}
	\right)\\
	= &   O\left( \sqrt{\frac{d}{n\sigma_x^2}} \sigma_e \right)   +  O\left(   
	\sqrt{ {\frac{\sigma_e^2}{\sigma_x^2} +  \Rcal_{1}^2 }}
	\sqrt{  \frac{\sigma_e^2}{\sigma_x^2}  + 1} 
	\right).
	\end{align*}
\end{proof}

\begin{proof}[Proof of Theorem \ref{thm:2}]
	We study the property of the estimator  in (\ref{eqt:label_embedding}) for each label $k\in\Lcal$. According to the definition,  we know that for each intermediate node $j\in \Mcal_k$, there exists a sample connected to $j$ that belongs to label $k$, and we denote this sample to be $\xb_{(j,0)}$. On the other hand, let $\xb_{(j, \hat{\eb}_k)}$ be the sample connected to node $j$ that is closest to $\hat{\eb}_k$. Let $\Scal_k\subseteq \Mcal_k$ be the set of intermediate nodes with $\xb_{(j, \hat{\eb}_k)} = \xb_{(j,0)}$. We have:
	\begin{align}
		\sum_{j\in \Mcal_k } \langle \hat{\eb}_k, \xb_{(j, \hat{\eb}_k)} \rangle = & \sum_{j\in \Scal_k} \langle \hat{\eb}_k, \xb_{(j, 0)}\rangle + \sum_{j\in \Scal_k^C} \langle \hat{\eb}_k,\xb_{(j, \hat{\eb}_k)} \rangle \\
		 \ge &  \sum_{j\in \Mcal_k } \langle \eb_k^\star, \xb_{(j, 0)} \rangle, 
	\end{align}
	where the first equality follows by definition, and the inequality holds because of the optimality of the estimator. Reorganizing both sides of the inequality, we get
	\begin{align}
		\sum_{j\in \Scal_k} \langle \hat{\eb}_k, \xb_{(j, 0)}\rangle + \sum_{j\in \Scal_k^C} \langle \hat{\eb}_k,\xb_{(j, \hat{\eb}_k)} \rangle
		\ge &  \sum_{j\in \Mcal_k} \langle \eb_k^\star, \xb_{(j,0)} \rangle \\ 
		\sum_{j\in \Scal_k} \langle \hat{\eb}_k, \xb_{(j, 0)} \rangle \ge&  \sum_{j\in \Scal_k} \langle \eb_k^\star, \xb_{(j, 0)}\rangle + \sum_{j\in \Scal_k^C} \left( \langle \eb_k^\star, \xb_{(j,0)} \rangle - \langle \hat{\eb}_k, \xb_{(j,\hat{\eb}_k)} \rangle \right). 
	\end{align}
	Now, let us use the definition of $\xb_{(j, 0)}$, and the fact that $\langle \hat{\eb}_k,\xb_{(j, \hat{e}^\prime)}\rangle \le 1, \forall j,\hat{e}^\prime$. We have:
	\begin{align}
		\sum_{j\in \Scal_k} \left( \langle \hat{\eb}_k, \eb_k^\star \rangle + \langle \hat{\eb}_k, -\epsilonb_{(j,0)} \rangle \right) \ge & \sum_{j\in \Scal_k}  \langle \eb_k^\star, \eb_k^\star  -\epsilonb_{(j, 0)}\rangle + \sum_{j\in \Scal_k^C} \left( \langle \eb_k^\star, \eb_k^\star  -\epsilonb_{(j, 0)} \rangle - 1 \right).
	\end{align}
	Rearrange the terms and normalize by $|\Scal_k|$, we have:
	\begin{align} 
		\langle \hat{\eb}_k, \eb_k^\star\rangle \ge & 1 - \frac{1}{|\Scal_k|} \sum_{j\in \Mcal_k} \langle \eb_k^\star, \epsilon_{(j,0)}\rangle + \frac{1}{|\Scal_k|} \sum_{i\in \Scal_k} \langle \hat{\eb}_k,\epsilonb_{(j, 0)}\rangle, 
	\end{align}
	which gives us the following: 
	\begin{align} 
		\langle \hat{\eb}_k, \eb_k^\star\rangle \ge & 1 - \frac{1}{|\Scal_k|} \sum_{j\in \Mcal_k} \langle \eb_k^\star, \epsilon_{(j, 0)}\rangle + \frac{1}{|\Scal_k|} \sum_{j\in \Scal_k} \langle \eb_k^\star ,\epsilonb_{(j,0)}\rangle + \langle \hat{\eb}_k-\eb_k^\star ,\epsilonb_{(j,0)}\rangle.
	\end{align}
	Based on the definition of $f(\cdot)$ in Definition \ref{def:1},
	\begin{align}
		\langle \hat{\eb}_k, \eb_k^\star\rangle \ge &  1 -  \frac{|\Scal_k^C|}{|\Scal_k|} f(|S_k^C|/|\Mcal_k|) -  {f(|\Scal_k|/|\Mcal_k|)} \left\| \hat{\eb}_k - \eb_k^\star \right\|  \\
		\langle \hat{\eb}_k, \eb_k^\star\rangle + f(1) \left\| \hat{\eb}_k - \eb_k^\star \right\| \ge & 1 -  \frac{1-\Delta}{\Delta} f(1-\Delta).
	\end{align} 
	We next show the minimum value $\Delta=|\Scal_k|/|\Mcal_k|$ for the estimator $\hat{\eb}_k$. By definition, all true embeddings are separated by at least $\delta$, and samples from other labels at most counts for $q$ proportion among all groups. Then, $\Delta \le \alpha$ means at least $(1-\alpha)|\Mcal_k|$ groups match to other labels, which means that at least $(1-\alpha - q) |\Mcal_k|$ samples do not come from the second label. As a result, the maximum value is upper bounded by:
	\begin{align}
		|\Mcal_k| - (1-\alpha - q)|\Mcal_k| \delta +  f(1) |\Mcal_k|,
	\end{align}
	while the lower bound for $\eb^\star$ is
	\begin{align}
		|\Mcal_k| -  f(1)|\Mcal_k|.
	\end{align}
	We can find out that if $\alpha \le 1- q - \frac{2  f(1) }{\delta}$, then we get contradictory. As a result, $\Delta > 1- q - \frac{2  f(1) }{\delta}$. 
	Plug in the property of $\Delta$ back to the above inequality, we have:
	\begin{align}
		\langle \hat{\eb}_k, \eb_k^\star\rangle +  f(1) \left\| \hat{\eb} - \eb^\star \right\| \ge   1 -  \left(\frac{1}{\Delta}-1\right) f(1-\Delta)  
		\ge   1 -  \left( \frac{1}{1-q - \frac{2  f(1)}{\delta} }-1 \right) f(1).
	\end{align}
	
	Re-organize the above inequality and use basic algebra, we get
	\begin{align}
		\langle \hat{\eb}_k, \eb_k^\star \rangle \ge & \min \left\{ 1-\varepsilon, \frac{1 -   \left( \frac{1}{1- q - \frac{2  f }{ \delta } }-1 \right) f  - \sqrt{2}  f }{ \left( 1-\sqrt{2}  f  \frac{1}{1-\sqrt{\varepsilon}} \right) } \right\} \\
	\Rightarrow \langle \hat{\eb}_k, \eb_k^\star \rangle	\ge & 1 - r   f  - (\sqrt{2} r + 2)  f ^2,
	\end{align}
	where $r = \left({1- q - \frac{2  f }{\delta } }\right)^{-1}-1$.
\end{proof}

\begin{proof}[Proof of Theorem \ref{thm:3}]
	For conciseness, we ignore the subscript $k$ in the following proof. 
	Our goal is to characterize the behavior of $\similarity{\eb_k^\star}{\eb_{t+1}}$. 
	Define sample index function 
	\begin{align}
	\Ical(j, \eb_t): = \argmax_{i \in \Ncal_j} \similarity{\eb_t}{\xb_i},
	\end{align}
	as the index of the sample selected by $\eb_t$ in  group $j$ (for notation simplicity, we assume this instance is unique). Furthermore, let
	\begin{align}
	\Scal_t^\good = \left\{ j\in \Mcal \mid \Tcal\left( \Ical(j, \eb_t)\right) = k  \right\}, \quad \Scal_t^\bad = \left\{j\in \Mcal \mid \Tcal\left( \Ical(j, \eb_t)\right) \neq k \right\}.
	\end{align} 

	Now, we can express the next iterate as: 
	\newcommand{\syyaa}{\sum_{j\in \Scal_t^\good}\xb_{(j, \eb_t)} + \sum_{i\in \Scal_t^\bad}\xb_{(j, \eb_t)} }
	\begin{align}
	\eb_{t+1} = & \frac{ \syyaa }{ \norm{\syyaa} } \\
	\similarity{\eb^\star}{\eb_{t+1}} = & \frac{  \sum_{j\in \Scal_t^\good} \similarity{\eb^\star}{\xb_{(j, \eb_t)} }  + \sum_{j\in \Scal_t^\bad} \similarity{\eb^\star}{\xb_{(j, \eb_t)} } }{\norm{\syyaa}} \\
	= & \frac{  \sum_{j\in \Scal_t^\good} \similarity{\eb^\star}{\eb^\star + \epsilonb_{(j, \eb_t)}}  + \sum_{j\in \Scal_t^\bad} \similarity{\eb^\star}{\xb_{(j, \eb_t)} } }{\norm{\syyaa}} \\
	= & \frac{ |\Scal_t^\good| +   {\color{black}\sum_{j\in \Scal_t^\good} \similarity{\eb^\star}{ \epsilonb_{(j,0)}} } + \sum_{j\in S_t^\bad} \similarity{\eb^\star}{\xb_{(j, \eb_t)} } }{\norm{\syyaa}}. \label{eqt:1}
	\end{align}
	According to the property of $\Scal_t^\bad$, 
	\begin{align}
	&\similarity{\xb_{(j, \eb_t)} }{\eb_t} \ge \similarity{\xb_{(j, 0)} }{\eb_t} = \similarity{\eb^\star}{\eb_t} + {\color{black} \similarity{\epsilonb_{(j, 0)}}{\eb_t}  }.
	\end{align}
	Therefore, 
	\begin{align} 
	\similarity{\eb^\star}{\xb_{(j, \eb_t)} } =& \similarity{\eb_t}{\xb_{(j, \eb_t)} } + \similarity{\eb^\star - \eb_t}{\xb_{(j, \eb_t)} } \\
	\ge& 
	\similarity{\eb^\star}{\eb_t} + {\color{black}{\similarity{\epsilonb_{(j,0)}}{\eb_t}} } +  \similarity{\eb^\star - \eb_t}{\xb_{(j, \eb_t)} } \\
	\ge& \similarity{\eb^\star}{\eb_t}  - \norm{\eb^\star - \eb_t} + {\color{black}{\similarity{\epsilonb_{(j, 0)}}{\eb_t} } }.
	\end{align}
	
	Plug in the result into (\ref{eqt:1}), 
	we have
	\begin{align}
	&\similarity{\eb^\star}{\eb_{t+1}}\\
	\ge& \frac{|\Scal_t^\good| + {\color{black} \sum_{j\in \Scal_t^\good} \similarity{\eb^\star}{ \epsilonb_{(j, 0)}}  }  + \sum_{j\in \Scal_t^\bad} \left( \similarity{\eb^\star}{\eb_t}  - \norm{\eb^\star - \eb_t} + {\color{black}{\similarity{\epsilonb_{(j, 0)}}{\eb_t}} } \right) }{\norm{\syyaa} } \\
	\ge & \frac{ |\Scal_t^\good | + |\Scal_t^\bad|  \left( \similarity{\eb^\star}{\eb_{t}} - \norm{\eb^\star - \eb_t}\right) + {\color{black} \sum_{j\in \Scal_t^\good} \similarity{\eb^\star}{ \epsilonb_{(j, 0)}}  } + {\color{black} \sum_{j\in \Scal_t^\bad} {\similarity{\epsilonb_{(j, 0)}}{\eb_t}} } }{ n}.
	\end{align}
	Denote $|\Scal_t^\good| /n = \alpha_t$, as a result, we have: 
	\begin{align}
	\similarity{\eb^\star }{\eb_{t+1}} \ge & \alpha_t + (1-\alpha_t) \left( \similarity{\eb^\star}{\eb_{t}} - \norm{\eb^\star - \eb_t}\right) + { \color{black} \frac{\sum_{j\in S_t^\good} \similarity{\eb^\star}{ \epsilonb_{(j, 0)}} + \sum_{i\in S_t^\bad} {\similarity{\epsilonb_{(j, 0)}}{\eb_t}} }{ n} } \label{eqt:1-step-app} \\
	\ge & \alpha_t + (1-\alpha_t) \left( \similarity{\eb^\star}{\eb_{t}} - \norm{\eb^\star - \eb_t}\right) -\alpha_t f(\alpha_t) - (1-\alpha_t) f(1-\alpha_t) \label{eqt:1-step-2-app} \\
	\ge & \alpha_t + (1-\alpha_t) \left( \similarity{\eb^\star}{\eb_{t}} - \norm{\eb^\star - \eb_t}\right) -f.
	\end{align}
	
\end{proof}

\end{document}